\documentclass{article}
\pdfoutput=1
\makeatletter
\renewcommand*{\@fnsymbol}[1]{\ifcase#1\or*\else\@arabic{\numexpr#1-1\relax}\fi}
\makeatother

\PassOptionsToPackage{numbers}{natbib}

\usepackage[final]{neurips_2024}

\usepackage[utf8]{inputenc} % allow utf-8 input
\usepackage[T1]{fontenc}    % use 8-bit T1 fonts
\usepackage{hyperref}       % hyperlinks
\usepackage{url}            % simple URL typesetting
\usepackage{booktabs}       % professional-quality tables
\usepackage{amsfonts}       % blackboard math symbols
\usepackage{nicefrac}       % compact symbols for 1/2, etc.
\usepackage{microtype}      % microtypography
\usepackage{xcolor}         % colors

\usepackage{amsmath}
\usepackage{amssymb}
\usepackage{mathtools}
\usepackage{amsthm}
\usepackage{mdframed}
\usepackage{subcaption}
\usepackage[capitalize,noabbrev]{cleveref}

\usepackage[textsize=tiny]{todonotes}
\usepackage{comment}

\usepackage{amsmath,amsfonts,bm}
\usepackage{amsthm}

\newtheorem{theorem}{Theorem}[section]

\newtheorem{lemma}[theorem]{Lemma}

\theoremstyle{definition}
\newtheorem{definition}[theorem]{Definition}

\theoremstyle{remark}
\newtheorem{remark}[theorem]{Remark}

\def\eqref#1{equation~\ref{#1}}
\def\1{\bm{1}}

\DeclareMathAlphabet{\mathsfit}{\encodingdefault}{\sfdefault}{m}{sl}
\SetMathAlphabet{\mathsfit}{bold}{\encodingdefault}{\sfdefault}{bx}{n}

\newcommand{\R}{\mathbb{R}}

\usepackage{url}
\usepackage{color}
\usepackage{graphicx}
\usepackage{float}
\usepackage{multirow}
\newcommand{\XX}{X}
\newcommand{\rhoA}{\rho_{A}}
\newcommand{\rhox}{\rho_{\XX}}

\newcommand{\Ltr}{{\rhox}}
\newcommand{\Lt}{{L^2(\rhox)}}

\newcommand{\Op}{{T}}

\newcommand{\Tau}{\mathcal{T}}

\usepackage{datetime}
\newcommand{\Hh}{\mathbb{H}}

\newcommand{\PP}{\mathbb{P}}

\newcommand{\EE}{\mathbb{E}}
\usepackage{array}
\usepackage{colortbl} % Required for \cellcolor

\definecolor{lightgray}{gray}{0.9} % Adjust the shade of gray as needed

\title{Harnessing small projectors and multiple views for efficient vision pretraining}

\author{%
  Kumar Krishna Agrawal \thanks{Equal Contribution, $^{\dagger}$ Co-senior authorship,\; Correspondence: blake.richards@mcgill.ca}\; \thanks{UC Berkeley, CA, USA} \hspace{5mm}
  \And 
  Arna Ghosh \footnotemark[1]\; \thanks{Mila - Quebec AI Institute \& Computer Science, McGill University, Montr\'eal, QC, Canada}\;
  \And 
  Shagun Sodhani \thanks{Meta FAIR, Toronto, ON, Canada}
  \And
  Adam M. Oberman$^{\dagger}$  \thanks{Mila - Quebec AI Institute \& Mathematics and Statistics, McGill University, Montr\'eal, QC, Canada}
  \And 
  Blake A. Richards $^{\dagger}$ \footnotemark[3]\; \thanks{Neurology \& Neurosurgery and Montreal Neurological Institute, McGill University, Montr\'eal, QC, Canada}\; \thanks{CIFAR Learning in Machines \& Brains Program, Toronto, ON, Canada}
}

\begin{document}

\maketitle

\setcounter{footnote}{0}
\begin{abstract}
Recent progress in self-supervised (SSL) visual representation learning has led to the development of several different proposed frameworks that rely on augmentations of images but use different loss functions. 
However, there are few theoretically grounded principles to guide practice, so practical implementation of each SSL framework requires several heuristics to achieve competitive performance.
In this work, we build on recent analytical results to design practical recommendations for competitive and efficient SSL that are grounded in theory. 
Specifically, recent theory tells us that existing SSL frameworks are actually minimizing the same idealized loss, which is to learn features that best match the data similarity kernel defined by the augmentations used.
We show how this idealized loss can be reformulated to a functionally equivalent loss that is more efficient to compute.
We study the implicit bias of using gradient descent to minimize our reformulated loss function, and find that using a stronger orthogonalization constraint with a reduced projector dimensionality should yield good representations.
Furthermore, the theory tells us that approximating the reformulated loss should be improved by increasing the number of augmentations, and as such using multiple augmentations should lead to improved convergence.
We empirically verify our findings on CIFAR, STL and Imagenet datasets, wherein we demonstrate an improved linear readout performance when training a ResNet-backbone using our theoretically grounded recommendations.  
Remarkably, we also demonstrate that by leveraging these insights, we can reduce the pretraining dataset size by up to 2$\times$ while maintaining downstream accuracy simply by using more data augmentations. 
Taken together, our work provides theoretically grounded recommendations that can be used to improve SSL convergence and efficiency.
\end{abstract}

\section{Introduction}

%%ss.17.05.2024: In the introduction, we have two references to unsupervised learning and > 14 references to SSL or Self-supervised learning. Should we update the two references to unsupervised learning to self-supervised learning as well ?

Unsupervised representation learning, i.e., learning features without human-annotated labels, is critical for progress in computer vision. Modern approaches, grouped under the \textit{self-supervised learning (SSL)} umbrella, build on the core insight that similar images should map to nearby points in the learned feature space -- often termed as the \textit{invariance criterion}. Current SSL methods can be broadly categorized into contrastive and non-contrastive algorithms, based on whether they formulate their loss functions using negative samples or not, respectively. 
%A prominent subgroup among non-contrastive SSL methods is the family of Canonical Correlation Analysis (CCA) algorithms \citep{zhang2021canonical}, which includes BarlowTwins \citep{zbontar2021barlow} and VICReg \citep{bardes2021vicreg}. These methods aim to enforce orthogonality among the learned features, in addition to the invariance criterion, and they have been shown to achieve competitive performance on benchmark computer vision datasets. 

Despite this difference in their loss formulations, recent theoretical work has established an equivalence between the contrastive and non-contrastive SSL frameworks \citep{garrido2022duality}. 
This work shows that these different SSL formulations are ultimately minimizing a loss that encourages the learning of features that best match the data similarity kernel defined by the augmentations used.
However, this notion of theoretical equivalence holds only in the limit of ideal pretraining settings, i.e. with access to infinite data and compute budget, and the feature learning behavior of different SSL algorithms in practical scenarios is still not well understood. 
Therefore, researchers often use empirically driven heuristics that are theoretically ungrounded to design successful applications, such as (i) a high-dimensional projector head for non-contrastive SSL or (ii) the use of two augmentations per image \citep{balestriero2023cookbook}. 
% Although these heuristics help in practice, their theoretical underpinnings are unclear. 
Moreover, existing SSL algorithms are extremely data-hungry, relying on large-scale datasets \citep{russakovsky2015imagenet} or data engines \citep{oquab2023dinov2} to achieve good representations. 
While this strategy works exceptionally well in data-rich settings (like training on natural-images), it is not viable in data-constrained settings (like medical imaging), where samples are relatively scarce.
% These methods have become the preferred strategy for representation learning in several domains due to the lack of need for negative samples and their simple formulation. 
% However, despite the apparent simplicity of their loss functions, the behavior of this family of algorithms is not well understood. Therefore, researchers often use empirically driven heuristics to design successful applications, such as using (i) a high-dimensional projector head or (ii) two augmentations per image. Although these heuristics help in practice, their theoretical underpinnings are unclear.

%%ss.17.05.2024: Lets add reference for CCA. https://proceedings.neurips.cc/paper/2021/file/00ac8ed3b4327bdd4ebbebcb2ba10a00-Paper.pdf has several references on page 2 --> done!

%%ss.17.05.2024: Lets add references for the ``representation learning in several domains`` part --> deleted.

%%ss.17.05.2024: Lets add references for the ``researchers often use empirically driven heuristics to design successful applications`` part --> done

\begin{figure*}[t]
    \centering
    \includegraphics[width=0.9\textwidth]{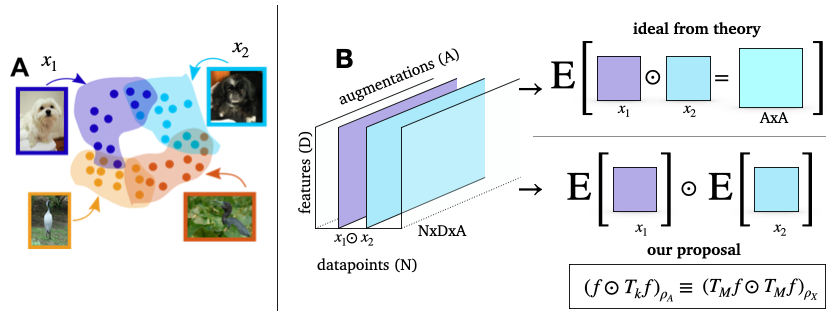}
    % \vspace{-0.5cm}
    \caption{  Design of existing SSL algorithms relies on heuristics. {\bf (A)} Augmentation graphs are common in vision pretraining, providing generalizable features for downstream tasks. (B) We propose an equivalent loss function for SSL pretraining that recovers the same eigenfunctions more efficiently than existing approaches. }
    \label{fig:schematic}
\end{figure*}

% Alongside relying on heuristics and researchers' intuition for design, existing SSL algorithms are extremely data-hungry. In particular, state-of-the-art algorithms often rely on large-scale datasets \citep{russakovsky2015imagenet} or data engines \citep{oquab2023dinov2} to achieve good representations. While this strategy works exceptionally well in natural-image settings, its application is limited in critical domains, such as medical imaging, where samples are relatively scarce.

%%ss.17.05.2024: re-wrote the last line: While this strategy works exceptionally well in data-rich settings (like training on natural-images), it is not viable in data-constrained settings (like medical imaging), where samples ar e relatively scarce.

With these challenges in mind, the primary focus of this work is to develop theoretically grounded recommendations for improving the effectiveness and efficiency of feature learning, both with respect to the required compute budget as well as data points.
% With these challenges in mind, the primary focus of this work is making progress toward establishing theoretical foundations for improving non-contrastive SSL algorithms (NC-SSL). 
Like any unsupervised representation learning algorithm, features learned through SSL depend on three factors: (i) implicit bias of the architecture, (ii) explicit invariance imposed by data augmentations, (iii) implicit bias of the learning rule.
While previous works predominantly studied the role of the model architecture capacity and loss function, and their interplay with data augmentations \citep{cabannes2023ssl,zhai2023understanding}, our approach broadens this perspective by also considering the role of the learning rule (gradient descent) in optimizing these loss functions.
% It is worth noting that ignoring the implicit bias of gradient descent has led researchers to incorrectly believe that high-dimensional projectors are necessary or that using multiple augmentations would not compensate for lesser samples.
% In particular, we analyze the BarlowTwins and VICReg losses and show that they implicitly learn the data similarity kernel defined by the chosen augmentations. We show that learning the data similarity kernel is helped by greater orthogonality in the projector outputs and more data augmentations. 
% First, in-line with previous findings \citep{zhai2023understanding}, we show that BarlowTwins and VICReg losses implicitly learn the data similarity kernel defined by the chosen augmentations.
% Second, we show that the ability to estimate the data similarity kernel is helped by more data augmentations.
Specifically, we extend the previous theoretical findings \citep{zhai2023understanding} that unified the desiderata of different SSL algorithms. We reformulate the idealized unifying loss to propose a functionally equivalent loss that is more compute-efficient (see \cref{fig:schematic}). Based on our loss formulation, we provide two practical recommendations that can help improve the efficiency of SSL pipelines while maintaining good performance. First, we show that optimizing the reformulated loss using gradient descent can often reduce the orthogonality among the learned embeddings, thereby leading to an inefficient use of the projector network's capacity. Consequently, we recommend using a stronger orthogonalization constraint to eliminate the requirement of high-dimensional projector heads, thereby significantly reducing the parameter overhead of good feature learning.
Second, we show that increasing the number of augmentations leads to a better estimate of the data similarity kernel. Consequently, we recommend using more augmentations to improve optimization convergence and learn better features earlier in training.
%%ss.17.05.2024: We should clarify what does ``more data augmentations`` mean. Is it more (higher number of) augmentations or more diverse augmentations -- done.

% Furthermore, we show that optimizing the loss using gradient descent can often reduce orthogonality in the projector outputs, thereby leading to inefficient use of the projector output dimensionality. 

%%ss.17.05.2024: Should we reword ``projector output dimensionality`` to ``projector's capacity`` in the last line -- done.

% As such, increasing the orthogonality of the projector output eliminates the requirement for a high-dimensional projector head, and increasing the number of data augmentations decreases the number of unique samples required. 

We empirically verify our theoretically grounded recommendations using the popular ResNet backbone on benchmark datasets: CIFAR, STL and Imagenet. 
Strikingly, we show that our multi-augmentation approach can learn good features even with \textit{half} of the samples in the pretraining dataset. 
% Our observations suggests that SSL pretraining can be done with smaller datasets and opens exciting questions in the design of performance-enhancing transformations. 
Our recommendations provide a path towards making SSL pretraining more data and compute-efficient without harming performance and could unlock massive performance gains in data-constrained setups.
%%ss.17.05.2024: Did any prior reviewers raised any concerns about the models potentially overfitting. 
%%ss.17.05.2024: Reworded the last line: Our observations suggests a path towards making SSL training more data and compute efficient and could unlock massive performance gains in data-constrained setups -- done! :)
In summary, our core contributions are as follows:
\begin{itemize}
    % \item \textbf{Implicit bias of loss and optimization:} We confirm that the loss functions of the CCA family of non-contrastive SSL algorithms are equivalent to the objective of learning eigenfunctions of the augmentation-defined data kernel. We use this fact to identify the implicit bias of gradient descent in learning particular eigenfunctions over others.
    % \item \textbf{Role of heuristics:} We provide a mechanistic explanation for the role of projector dimensionality and the number of data augmentations and empirically demonstrate that low-dimensional projector heads are sufficient and that using more augmentations leads to learning better representations.
    \item \textbf{Efficient SSL loss formulation:} We propose an functionally equivalent and compute-efficient formulation of the SSL desiderata that yields the eigenfunctions of the augmentation-defined data similarity kernel.
    \item \textbf{Role of heuristics:} Based on our loss formulation and the implicit bias of gradient descent in optimizing this loss, we provide a mechanistic explanation for the role of projector dimensionality and the number of data augmentations. Consequently, we empirically demonstrate that low-dimensional projector heads are sufficient and that using more augmentations leads to learning better representations.
    \item \textbf{Data efficient SSL:} Leveraging the convergence benefits of the multi-augmentation SSL framework, we empirically demonstrate that we can learn good features with significantly smaller datasets (up to $2\times$) without harming downstream linear probe performance.

    %%ss.17.05.2024: If I understsood correctly, the last line can be re-worded to say "we can learn good features with significantly smaller datasets (up to 4\times) without harming downstream performance."   
\end{itemize}

\section{Preliminaries}

{\bf Existing SSL approaches in computer vision}
In recent years machine learning researchers have developed a number of effective approaches for learning from data without labels. The most popular approaches use augmentations of data points as targets for themselves. One of the first was a Simple framework for Contrastive Learning (SimCLR), which relied on an infoNCE loss with augmentations of an image as positive targets and augmentations of other images as negative samples (to construct the contrastive loss) \citep{chen2020simple}. Other works have relied on non-contrastive approaches, notably BarlowTwins \citep{zbontar2021barlow} and VICReg \citep{bardes2021vicreg}. BarlowTwins, which was inspired by the ideas of the neuroscientist Horace Barlow (cite), also uses augmentations of images, but it instead aims to optimize the covariance structure of the representations in order to reduce redundancies in the feature space \citep{zbontar2021barlow}. Variance Invariance Covariance Regularization (VICReg) was a modification of BarlowTwins that added a variance term in the loss in order to ensure that every feature dimension has a finite variance \citep{bardes2021vicreg}. In this paper we will focus on non-contrastive methods like BarlowTwins and VICReg, but in line with previous work \citep{zhai2023understanding}, we also consider how these approaches relate to contrastive methods like SimCLR.   

{\bf Formalizing the self-supervised learning problem}
Now, we will formalize the unsupervised representation learning problem for computer vision. 
In particular, we assume access to a dataset 
%$\mathcal{D} = \{ \mathrm{x_1, x_2, ..., x_n } \}$
$\mathcal{D} = \{ {x_1, x_2, ..., x_n } \}$
with $x_i \in \mathbb{R}^p$ consisting of unlabeled images.
%%ss.17.05.2024: In the above line, can we just say ``consisting of unlabeled images`` ? -- done
The objective is to learn a $d$-dimensional representation ($d < p$) that is useful across multiple downstream applications. 
We focus on learning the parameters of a deep neural network $f_\theta \in \mathcal{F}_\Theta$, using the multi-augmentation SSL framework, wherein multiple views of an image are used to optimize the pretraining loss function, $\mathcal{L}_{pretrain}(f_\theta, \mathcal{D})$
%%ss.17.05.2024: ``Multi-view invariance SSL framework`` should also be mentioned in the Introduction of the paper. It is in the title of the paper but is not even mentioned in the introduction. -- updated

{\bf Non-Contrastive Self-Supervised Learning} 
(NC-SSL) algorithms impose invariance to data augmentations, while imposing regularization on the geometry of the learned feature space. 
More generally, $\mathcal{L}_{pretrain}$ can be formulated with two terms (i) $\mathcal{L}_{invariance}$: to learn invariance to data augmentations and (ii) $\mathcal{L}_{collapse}$: regularization to prevent collapsing the feature space to a trivial solution.
\begin{equation}
    \mathcal{L}_{pretrain} := \mathcal{L}_{invariance} + \beta \mathcal{L}_{collapse}
    \label{eq:SSL2}
\end{equation}
where $\beta$ denotes a hyperparameter that controls the importance of the collapse-preventing term relative to the invariance term.
This formulation separates features that are invariant to the augmentations from those that are sensitive to them.
Intuitively, the ideal feature space is more sensitive to semantic attributes (e.g. ``that's a dog'') and less sensitive to irrelevant attributes (e.g. ``direction the dog is facing''), facilitating generalization to new examples. 

{\bf Data Augmentation graph} was introduced by Haochen \textit{et al.} \citep{haochen2021provable} to analyze contrastive losses, like SimCLR \citep{chen2020simple}. 
Briefly, we define a graph $\mathcal{G(A, W)}$ that captures the relationship between images derived from all possible data augmentations.
The vertex set ($\mathcal{A},\rhoA$) is each augmented sample in a dataset, $\mathcal{X}$, and the adjacency matrix, $\mathcal{W}$, denotes the similarity between pairs of vertices. 
Let $x_0$ be an image in $\mathcal{X}$, and let $z = M(x_0) \in \mathcal{A}$ be a random data augmentation of the image, $x_0$.
We define the probability density of reaching $z$ from $x_0$ via a choice of mapping $M$: 
\begin{align}
    p(z \mid  x_0) = \PP( z = M(x_0)),
\end{align}
Since the mapping is not generally invertible (e.g., cropping), we observe that $p(x_0 \mid z) \not = p(z \mid x_0)$. 
Using this definition, we now formally define the strength of the edge between nodes $x,z \in \mathcal{A}$ of the augmentation graph as the joint probability of generating augmentations $x,z$ from the same image $x_0 \sim \rhox$. Notably, the edge strength of the (degree-normalized) augmentation graph is equivalent to the data similarity kernel, defined in \citep{zhai2023understanding}.
Formally,
\begin{equation}
    k^{DAF}(x,z) = w_{xz} := \EE_{x_0 \sim \rhox} \left[ \frac{p(x \mid x_0)}{p(x)} \frac{p(z \mid x_0)}{p(z)} \right]
    % w_{xz} := \mathbb{E}_{x_0 \sim  \rhox} \left[ p(x \mid  x_0)p(z \mid  x_0) \right] 
\end{equation}
The magnitude of $w_{xz}$ captures the augmentation-defined similarity between $x$ and $z$. A higher value of $w_{xz}$ indicates that both patches are more likely to come from the same image and, thereby, are more similar. 

% \begin{figure*}[t]
%     \centering
%     \includegraphics[width=0.9\textwidth]{figs/FinalFigs/augmentation_graph_schematic.png}
%     \caption{Schematic of augmentation graph. (A) Augmentations from each image span a region in the image space, which could overlap with the augmentation span of other images. (B) An augmentation graph schematic that uses probabilities to characterize the interactions among augmentation spans of different instances.}
%     \label{fig:aug_graph}
% \end{figure*}

%ss.17.05.2024: "Augmentations from each image span a region in the image space, which could overlap with the augmentation span of other images." seems like a strong claim. We are essentially claiming that we can create other images (in the dataset) by just applying data-augmentations. This could be true for mnist like datasets in the image space but is very likely not true for other, more complex datasets in the image space. It could be true for these more complex datasets in the latent space but very likely not in the image space / pixel space. 

% \citep{haochen2021provable} showed that optimizing a functionally equivalent form of the SimCLR loss termed the spectral contrastive loss ($\mathcal{L}_c$), essentially learns features whose covariance structure matches the adjacency matrix of the augmentation graph:
The desiderata of different SSL algorithms can be understood as learning features $F$ that best capture $k^{DAF}(x,z)$, i.e. $F(x)^TF(z) \approx k^{DAF}(x,z)$. Recent theoretical work has shows that different SSL losses can be formulated as special cases of the objective function that recovers the top-d eigenfunctions of $k^{DAF}(x,z)$ \citep{zhai2023understanding}.

\begin{equation}
    \mathcal{L}_{ssl}(F) = \mathbf{E}_{x,z \in \mathcal{A}}\left[(k^{DAF}(x,z) - F(x)^TF(z))^2 \right]
    \label{eq:zhai_loss_simplified}
\end{equation}

% \begin{equation}
%     \mathcal{L}_c \propto \|ZZ^T - \bar{\mathcal{W}} \|^2_F 
%     \label{eq:haochen_loss}
% \end{equation}
% where $Z$ denotes the output of the neural network, $\bar{\mathcal{W}}$ denotes the degree-normalized adjacency matrix and $\| . \|_F$ denotes the Frobenius norm operator. 
% This implies that the features learned by a contrastive SSL framework would align with the top eigenvectors of $\bar{\mathcal{W}}$. 
% As observed by \citep{haochen2021provable}, all rotations of $Z$ that don't change its span define an equivalence class of solutions to the above optimization problem and make no difference for the downstream generalization of a linear probe. 
% Furthermore, \citep{zhai2023understanding} extended this result for general SSL settings (including the NC-SSL), showing that learned features align with the eigenvectors of the data augmentation covariance kernel (equivalent to $\bar{\mathcal{W}}$).
% Note that all rotations of $Z$ that don't change its span define an equivalence class of solutions to \Cref{eq:haochen_loss} and make no difference for the downstream generalization of a linear probe. 
% Based on this insight, we define an equivalence among learned feature spaces:

Note that all rotations of $F$ that don't change its span define an equivalence class of solutions to \Cref{eq:zhai_loss_simplified} and make no difference for the downstream generalization of a linear probe. 
Based on this insight, we define an equivalence among learned feature spaces:
%%ss.17.05.2024: "be a $d$-dimensional feature vector (a vector of functions)." can be re-written as "be a $d$-dimensional function vector (a vector of functions)."
\begin{definition}\label{defn.equivalent}
Let $F(x) = (f_1(x), \dots f_d(x))$ be a $d$-dimensional feature vector (a vector of functions).  Define the subspace 
\begin{equation}\label{VofF}
	V = V(F) = \{  h: \XX \to \R  \mid h(x) = w\cdot F(x),\quad w\in \R^d \}
\end{equation}
to be the span of the components of $F$. Given an $n$-dimensional feature vector,  $G(x) = (g_1(x), \dots, g_n(x))$  we say the features $G$ and $F$ are equivalent, if $V(F) = V(G)$. 
\end{definition}

% {\bf Data similarity kernel}, a general formulation of the augmentation graph adjacency matrix, can be used to denote the similarity between two images. 
% As defined in \citep{zhai2023understanding}, there are two notions of similarity kernel, namely the forward and backward data augmentation kernel, for images $x,z \in \mathcal{A}$ and $x,z \in \mathcal{X}$, respectively. 
% \begin{align}\label{kDA}
%     k^{DAF}(x,z) &= \EE_{x_0 \sim \rhox} \left[ \frac{p(x \mid x_0)}{p(x)} \frac{p(z \mid x_0)}{p(z)} \right] \nonumber\\
%     k^{DAB}(x,z) &= \EE_{x_0 \sim \rhoA} \left[ \frac{p(x_0 \mid x)}{p(x_0)} \frac{p(x_0 \mid z)}{p(x_0)} \right] 
% \end{align}

% Intuitively, the forward covariance kernel ($k^{DAF}$) measures the similarity between $x,z$ in terms of how likely they can be reached from $x_0 \in \mathcal{X}$, weighted by the distribution of $x_0$, and vice-versa for the backward covariance kernel ($k^{DAB}$).

% \section{Data augmentation kernel perspective of non-contrastive SSL}
\section{Implicit bias of non-contrastive SSL loss and optimization}

% We extend the current theoretical understanding of NC-SSL, particularly BarlowTwins and VICReg, building upon recent theoretical advancements that posit learned features as spanning the eigenfunctions of the augmentation-defined data covariance operator, $\Tau$ (defined below, see \citep{cabannes2023ssl,zhai2023understanding} for details). 
% Our analysis reveals that under weak orthogonalization constraints, learning dynamics imposed by gradient descent biases the feature space to span dominant eigenfunctions of $\Tau$, i.e. corresponding to larger eigenvalues. 
% This leads to an intriguing insight: features learned using a high-dimensional projector more accurately reflect $\Tau$ at the end of training, compared to using a low-dimensional projector, owing to a better alignment with all eigenvectors of $\Tau$ at initialization. 
We extend the recent theoretical results \citep{zhai2023understanding} to propose a compute-efficient reformulation of the loss function of the SSL desiderata that yields equivalent features, i.e. the functions spanning the eigenfunctions of the augmentation-defined data similarity kernel, $k^{DAF}$. Furthermore, we study the role of gradient descent in optimizing this loss function and uncover a selection and primacy bias in feature learning. Specifically, we find that gradient descent tends to learn the dominant eigenfunctions (eigenfunctions corresponding to larger eigenvalues) earlier during training, and often over-represents these eigenfunctions under weak orthogonalization constraints.

Consequently, we propose employing a stronger orthogonalization constraint during optimization when using a low-dimensional projector to ensure that learned features are equivalent to those learned with a high-dimensional projector.
Furthermore, we argue that using more augmentations improves our sample estimate of $k^{DAF}$, thereby aiding the eigenfunction optimization problem. 
We dedicate the rest of this section to highlight our key theoretical insights, and practical recommendations that follow them.
% Therefore, we propose using an averaging operator with more data augmentations to learn features that better optimize the loss function.
% This holistic approach, integrating both loss function characteristics and optimization dynamics, provides a more comprehensive understanding of the feature learning process during SSL.

% Following the previous section, we present an augmentation kernel perspective of BarlowTwins and VICReg losses. Specifically, we show that these losses are equivalent to the optimization problem of learning eigenfunctions of the augmentation-defined data covariance kernel. Subsequently, we argue that using a high-dimensional projector yields better overlap with the top eigenvectors of the data augmentation kernel at initialization compared to a low-dimensional projector. 
% Therefore, our analysis suggests using a stronger orthogonalization constraint during optimization for lower-dimensional projectors to ensure that features learned are equivalent to those learned with high-dimensional projectors.
% Furthermore, we argue that using more augmentations improves our estimate of the augmentation-defined data covariance kernel, thereby aiding the eigenfunction optimization problem.
% Therefore, our analysis suggests using an averaging operator with more data augmentations to estimate the true augmentation kernel better.

\subsection{Features in terms of data augmentation kernels}

Let us define a kernel operator, $T_k$, for a positive semi-definite data augmentation kernel, $k^{DAF}$.
\begin{equation}\label{Tkdefn}
	T_k f(x) =\EE_{z\sim \rhox} [ k(z,x) f(z) ]
\end{equation}
such that \cref{eq:zhai_loss_simplified} can be equivalently written as (Equation 5 of \cite{zhai2023understanding})
\begin{equation}
    \mathcal{L}_{ssl}(F) = \langle F, (I - T_k) F \rangle_{\rho_\mathcal{A}} 
    \label{eq:zhai_loss}
\end{equation}
We can now use Mercer's theorem to factorize $k^{DAF}$ into corresponding spectral features $G: \XX \to \ell_2$ (where $\ell_2$ represents square summable sequences) \cite{deng2022neural,deng2022neuralef,pfau2018spectral}. However, note that computing $k^{DAF}$ (or $T_k$) is expensive as it requires computing the overlap among all augmentations of every pair of data points.
Instead of computing the eigenfunctions of $T_k$ directly, we propose using an alternative operator $T_M$:
\begin{equation} \label{TDAf}
	\Op_M f(x) 	= \EE_{x_0 \sim M(x)} \left[f(x_0) \right] = \sum_{x_0} \left[ p(x_0 \mid x) f(x_0) \right ]
\end{equation}
which averages the values of the function, $f$, over the augmented images $x_0 = M(x)$ of the data, $x$. 
We show that $T_M^TT_M$ is equivalent to $T_k$, and therefore $T_M$ and $T_k$ have shared eigenfunctions.
% $T_M^TT_M$ can be thought of as the augmentation-defined covariance operator, $\Tau$, for a choice of $M(.)$, and as we show below, it provides a more efficient estimate of $T_k$.

% Since the operator $T_k$ is compact and positive, it has a spectral decomposition consisting of eigenfunctions $\phi_i$ and corresponding eigenvalues $\lambda_i$. Using these eigenpairs, we can define the (infinite sequence of square summable) spectral features, $G: \XX \to \ell_2$, (where $\ell_2$ represents square summable sequences), by 
% \begin{equation}
% 	\label{SpectralFeatures}
% 	G(x) =
% 	 (\sqrt{\lambda_1}\phi_1(x), \dots, \sqrt{\lambda_d}\phi_d(x), \dots )
% \end{equation}
% Then, Mercer's theorem gives
% \begin{equation}\label{Mercer}\tag{Mercer}
% 	k^{DA}(x,z)
% 	=G(x)\cdot G(z)
% \end{equation}
% and ensures that the inner product is finite. These are the desired features that factor in the kernel. 
% Instead of computing the eigenfunctions of $T_K$ directly, we propose using an alternative, more efficient, operator $T_M$ that leads to learning equivalent features (see \Cref{defn.equivalent})
% However, computing the eigenfunctions of $T_k$ is costly. Instead, we propose an alternative using the more efficient operator $T_M$. According to Definition~\ref{defn.equivalent}, both operators lead to equivalent features. 

\begin{theorem}
Let $G(x)$ be the infinite Mercer features of the backward data augmentation covariance kernels, $k^{DAB}$. Let $F(x) = (f_1(x), \dots, f_{N_k}(x))$ be the features given by minimizing the following data augmentation invariance loss
\begin{align}\label{eq:ourloss}
L(F) &= \sum_{i=1}^{N_k} \|T_M f_i - f_i\|^2_\Lt , \quad \text{subject to }\quad (f_i, f_j)_\Ltr = \delta_{ij} 
\end{align}
which includes the orthogonality constraint. Then, $V(F) \subset V(G)$ ,  $\lim_{N_k\to\infty} V(F) = V(G)$. 
%Moreover, the same result holds is we use $k^{DAF}$ instead of $k^{DAB}$.  
\label{theorem:eigen}
\end{theorem}

As shown in the \Cref{sec:kernel_SSL_theory}, $L(F)$ is equivalent to a constrained optimization formulation of the BarlowTwins loss. 
% Furthermore, VICReg loss can also be seen as an unconstrained optimization formulation of $L(F)$ with the additional constraint that $(f_i,f_i) \geq \gamma$ $\forall i \in \{1, 2 \dots N_k\}$. 
Furthermore, $L(F)$ with the additional constraint that $(f_i,f_i) \geq \gamma$ $\forall i \in \{1, 2 \dots N_k\}$ is the constrained optimization formulation of the VICReg loss. 

\subsection{The implicit bias of gradient descent}
Next, we investigate how the use of gradient descent for optimizing $L(F)$ influences the characteristics of the learned feature space, $V(F)$. Given the similarity in its form with that of the BarlowTwins loss, we build on recent findings that demonstrate the sequential nature of learning eigenfunctions when optimizing the BarlowTwins loss under a strong orthogonalization regularization \citep{simon2023stepwise}.
% \citep{simon2023stepwise} studied the learning dynamics of the BarlowTwins loss under a strong orthogonalization constraint and found that eigenfunctions are learned sequentially during training. 
%%ss.17.05.2024: Lets add a reference for " strong orthogonalization is seldom used due to instabilities in training." -- done
% However, in practice, strong orthogonalization is seldom used due to instabilities in training \citep{zbontar2021barlow,balestriero2023cookbook}. 
Since strong orthogonalization is seldom used in practice due to instabilities in training \citep{zbontar2021barlow,bardes2021vicreg}, we believe studying the learning dynamics under weak orthogonalization regularization (i.e. low values of $\beta$ in \cref{eq:SSL2}) is more relevant to provide recommendations for practitioners.
% Here, we present a relationship between $V(F)$ and the eigenfunctions of augmentation-defined covariance operator, $\Tau$, in weak orthogonalization cases. 
%%ss.17.05.2024: I dont think we defined what is ``weak orthogonalization`` -- done

% This characterization aims to provide an understanding of the feature learning dynamics when optimizing these losses under conditions that are closer to practice. 
% Informally,

\begin{theorem}
    (Informal) Let us denote the span of the feature space at initialization as $V(F_0)$ and after training as $V(F_T)$.
    For small initialization of the network's weights, the alignment of $V(F_T)$ with the eigenfunctions of $T_k$ depend on two factors: (i) alignment of $V(F_0)$ with the eigenfunctions of $T_k$; (ii) singular values of $T_k$.
    % \begin{enumerate}
    %     \item The alignment of $V(F_0)$ with the eigenfunctions of $\Tau$
    %     \item The singular values of $\Tau$
    % \end{enumerate}
    \label{theorem:dynamics}
\end{theorem}
Under weak orthogonalization constraints, the network tends to learn features that are strongly aligned with eigenfunctions corresponding to large singular values. 
We refer to this property as the ``selection'' bias of gradient descent, wherein gradient descent selects certain eigenfunctions based on the corresponding singular values. 
This selection bias leads to redundancy among the learned feature space, thereby reducing the effective dimensionality of the network's output space compared to its ambient dimensionality. 
We will leverage this finding to improve the parameter overhead of good feature learning using BarlowTwins and VICReg loss frameworks.
% This inclination towards a stronger alignment with dominant eigenfunctions leads to redundancy among the feature dimensions, thereby reducing the effective dimensionality of the network's output space compared to its ambient dimensionality. 
% Therefore, our findings highlight a pivotal aspect of feature space optimization that can significantly address the compute efficiencies in learning good features using BarlowTwins and VICReg losses.

\subsection{Takeaway 1: Low-dimensional projectors can yield good representations}
\label{cor:low_pdim}

Given the proximity of the formulation of \cref{eq:ourloss} to that of BarlowTwins and VICReg losses, we will leverage existing heuristics that have been shown to work in practice. 
As such, BarlowTwins and VICReg frameworks call for high-dimensional projectors while using a weak orthogonalization regularization to facilitate good feature learning. 
% However, our kernel perspective (\Cref{theorem:eigen}) challenges this notion. 
% However, our theoretical insights from \Cref{theorem:eigen} and \Cref{theorem:dynamics} challenge the requirement of a high-dimensional projector to learn the eigenfunctions of the underlying data similarity graph.
We know, from \Cref{theorem:eigen}, that the eventual goal of these frameworks is to learn the eigenfunctions of the underlying data similarity graph.
For example, since the intrinsic dimensionality of Imagenet is estimated to be $\sim 40$ \citep{pope2020intrinsic}, it is not unreasonable to expect that the span of desired features would be of similar dimensionality. 
It is, thus, intriguing that the current practice would suggest using an $\sim 8192$-dim projector head to capture the intricacies of the corresponding augmentation-defined data similarity kernel. 
This discrepancy can be explained by analyzing the learning dynamics, as in \Cref{theorem:dynamics}. 
% This discrepancy can be explained by observing the learning dynamics of a linearized model under the BarlowTwins loss optimization \citep{simon2023stepwise}. These dynamics reveal that initializing the projection weight matrix in alignment with the eigenfunctions of the data kernel retains this alignment throughout the learning process. 
Notably, a high-dimensional projector is likelier to have a greater initialization span than its low-dimensional counterpart, thereby increasing the alignment between $V(F_0)$ and relevant eigenfunctions of $T_k$.
We hypothesize that a stronger orthogonalization constraint for low-dimensional projectors can rectify this issue, reducing the redundancy in the network's output space and rendering it sufficient for good feature learning.

\subsection{Takeaway 2: Multiple augmentations improve kernel approximation}
\label{cor:multi_augs}
% \Cref{theorem:eigen} proposes a compute-efficient way formulation of the true SSL loss, and entails computing only the $T_M$ operator instead of the more expensive $T_k$ operator. However, computing $T_M$ also requires an expectation over different augmentations for each sample $x$.
% \Cref{theorem:eigen} implies that the invariance loss optimization would ideally entail using the $T_M$ operator, thereby requiring many augmentations for each sample $x$. 
By comparing the invariance criterion formulation in the standard BarlowTwins and VICReg losses to \Cref{eq:zhai_loss}, it can be inferred that current practices use a sample estimate of $T_k$.
Using only two augmentations per sample yields a noisy estimate of $T_k$, yielding spurious eigenpairs \citep{vershynin2010introduction} (see \Cref{sec:multi_augs_theory}). 
% In contrast, using only two augmentations per sample (as is common practice) yields a noisy estimate of $T_M$, yielding spurious eigenpairs \citep{vershynin2010introduction} (see \Cref{sec:multi_augs_theory}). 
% These spurious eigenpairs add stochasticity to the learning dynamics and hinder the alignment of the learned features with the eigenfunctions of the data kernel \citep{simon2023stepwise}.
These spurious eigenpairs add stochasticity in the learning dynamics, and coupled with \Cref{theorem:dynamics}, increase the redundancy in the learned feature space \citep{chen2023stochastic}.
We hypothesize that improving this estimation error by increasing the number of augmentations could alleviate this issue and improve the speed and quality of feature learning.

Of course, increasing the number of augmentations ($m$) in the standard BarlowTwins and VICReg loss improves the estimate of $T_k$ but comes with added compute costs -- a straightforward approach would involve calculating the invariance loss for every pair of augmentations, resulting in $\mathcal{O}(m^2)$ operations. 
However, \Cref{theorem:eigen} proposes an alternative method that uses the sample estimate of $T_M$, thereby requiring only $\mathcal{O}(m)$ operations, and hence is computationally more efficient while yielding functionally equivalent features (see \Cref{sec:kernel_SSL_theory}).
% Both these strategies are functionally equivalent (see \Cref{sec:kernel_SSL_theory}), but the latter is computationally more efficient. 
In summary, \Cref{theorem:eigen} establishes a mechanistic role for the number of data augmentations, paving the way for a computationally efficient multi-augmentation framework:
\begin{equation}
    \widehat{L}(F) = \EE_{x \sim \rhox} \left[\sum_{i=1}^{N_k} \sum_{j=1}^m \| \overline{f_i (x)} - f_i (x_j)\|^2_\Lt \right] ,\quad \text{subject to }\quad (f_i, f_j)_\Ltr = \delta_{ij}    
    \label{eq:ma_ncssl}
\end{equation}
% \begin{align}
%     \label{eq:ma_ncssl}
%     \widehat{L}(F) &= \EE_{x \sim \rhox} \left[\sum_{i=1}^{N_k} \sum_{j=1}^m \| \overline{f_i (x)} - f_i (x_j)\|^2_\Lt \right] \\
%     & \text{subject to:}\quad (f_i, f_j)_\Ltr = \delta_{ij} \nonumber   
% \end{align}

where $\overline{f_i (x)} = \frac{1}{m}\sum_{j=1}^m f_i(x_j)$ is the sample estimate of $T_M f_i(x)$.

\section{Experiments}

In our experiments, we seek to (i) provide empirical support for our theoretical insights and (ii) present practical primitives for designing efficient SSL routines. Since our proposed loss function is closest to the formulation of BarlowTwins/VICReg, we present empirical evidence comparing our proposal to these baselines. In summary, with extensive experiments across learning algorithms (BarlowTwins \& VICReg) and training datasets (CIFAR-10, STL-10 \& Imagenet-100), we establish the following:
\begin{itemize}
    \item {\bf low-dimensional projectors} can yield {\it good representations}.
    \item {\bf multi-augmentation} improves downstream accuracy, as well as convergence rate.
    \item multi-augmentation {\bf improves sample efficiency} in SSL pretraining, i.e., recovering similar performance with significantly fewer unique unlabelled samples.
\end{itemize} 

{\bf Experiment Setup}: We evaluate the effectiveness of different pretraining approaches using image classification as the downstream task. Across all experiments, we pretrain a Resnet feature encoder backbone for 100 epochs (see \Cref{sec:longer_pretraining} for longer pretraining results) and use linear probing on the learned representations\footnote{Code:
\href{https://github.com/kumarkrishna/fastssl}{https://github.com/kumarkrishna/fastssl}}. All runs are averaged over 3 seeds; error bars indicate standard deviation. Other details related to optimizers, learning rate, etc., are presented in the \Cref{sec:implementation}.

\subsection{Low-dimensional projectors can yield good representations}

\begin{figure}[!ht]
    \centering
    \includegraphics[width=\linewidth]{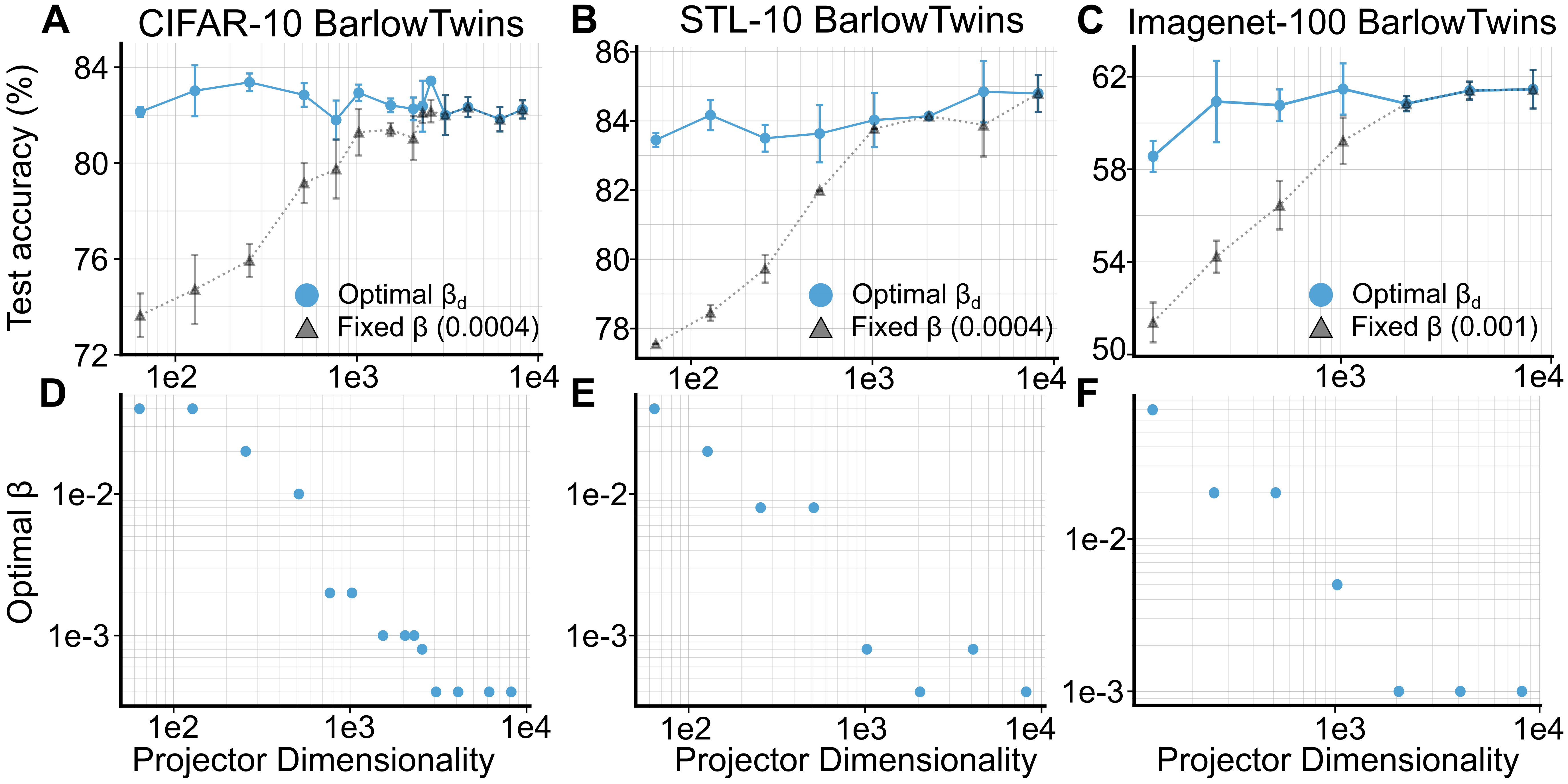}
    \vspace{-0.5cm}
    \caption{Low-dimensional projectors can yield good representations. We demonstrate that using a higher orthogonality constraint, $\beta$, for lower projector dimensionality can achieve similar performance over a wide range of projector dimensions ($d$).}
    \label{fig:fig3_lambda_pdim}
\end{figure}

%%ss.17.05.2024: Did any previous reviwers asked for experiments with larger values of number of augmentations?

\begin{table}[!ht]
 \small
  \centering
  \begin{tabular}{c|c|c|c|c}
    \hline
    \multicolumn{1}{c|}{\multirow{2}{*}{{\it pdim}}} & \multicolumn{2}{c|}{\bf Barlow Twins} & \multicolumn{2}{c}{\bf VICReg} \\
    % \cline{2-3}
    \multicolumn{1}{c|}{} & fixed $\beta$ & optimal $\beta^*$  & fixed $\beta$ & optimal $\beta^*$ \\
    \hline
    \vspace{-2mm} \\
    64 & $73.6 \pm 0.9$ & $82.1 \pm 0.2$ & $68.9 \pm 0.2$ & $81.9 \pm 0.1$ \\
    % 128 & $74.7 \pm 1.4$ & $83.0 \pm 1.1$ & $70.6 \pm 0.3$ & $82.3 \pm 0.4$ \\
    256 & $75.9 \pm 0.7$ & \cellcolor{lightgray}{\bf 83.4 $\pm$ 0.4} & $75.3 \pm 0.2$ & $81.9 \pm 0.3$ \\
    1024 & $81.3 \pm 1.0$ & $82.9 \pm 0.3$ & $79.2 \pm 0.9$ & \cellcolor{lightgray}{\bf 82.5 $\pm$ 0.9} \\
    8192 & $82.2 \pm 0.4$ & $82.2 \pm 0.4$ & $80.4 \pm 1.5$ & $80.4 \pm 1.5$  \\
    \hline
  \end{tabular}
  \caption{Optimizing for orthogonality appropriately allows low-dimensional projectors to match the performance (on CIFAR-10) of much higher-dimensional projectors.}
  \vspace{-5mm}
  \label{tab:multicol_example}
\end{table}

Existing works recommend using high-dimensional MLPs as projectors (e.g., d=8192 for Imagenet in \cite{zbontar2021barlow, bardes2021vicreg}), and show significant degradation in performance when using lower-dimensional projectors for a fixed redundancy coefficient ($\beta$). To reproduce this result, we run a grid search to find the optimal coefficient $(\beta^*_{8192})$ for $d=8192$ and show that performance progressively degrades for lower $d$ if the same coefficient $\beta^*_{8192}$ is reused for $d \in \{64, 128, 256, 512, 1024, 2048, 4096, 8192\}$. 

Our insights in \Cref{cor:low_pdim} suggest low-dimensional projectors should recover similar performance with appropriate orthogonalization. To test this, we find the best $\beta$ by performing a grid search independently for each $d \in \{64, 128, 256, 512, 1024, 2048, 4096, 8192\}$. As illustrated in \Cref{fig:fig3_lambda_pdim}, using low-dimensional projectors yield features with similar downstream task performance, compared to the features obtained using high-dimensional projectors. 
Strikingly, we also observe that the optimal $\beta_{d} \propto 1/d$, which aligns with our theoretical insights.

\begin{mdframed}
{\bf Recommendation:} 
Start with low-dimensional projector, using $\beta=\mathcal{O}(\frac{1}{d})$, and sweep over $(pdim=d, \beta=\mathcal{O}\left(\frac{1}{d}\right))$ if needed.
\end{mdframed}

%%ss.17.05.2024: For the ``sweep over`` par -- something missing here!

\subsection{Multiple Augmentations Improve Performance and Convergence}
    
Although some SSL pretraining approaches, like SWaV \citep{caron2018deep}, incorporate more than two views, the most widely used heuristic in non-contrastive SSL algorithms involves using two views jointly encoded by a shared backbone. In line with this observation, our baselines for examining the role of multiple augmentations use two views for computing the cross-correlation matrix.
\begin{figure}[htbp]
    \centering
    \includegraphics[width=\linewidth]{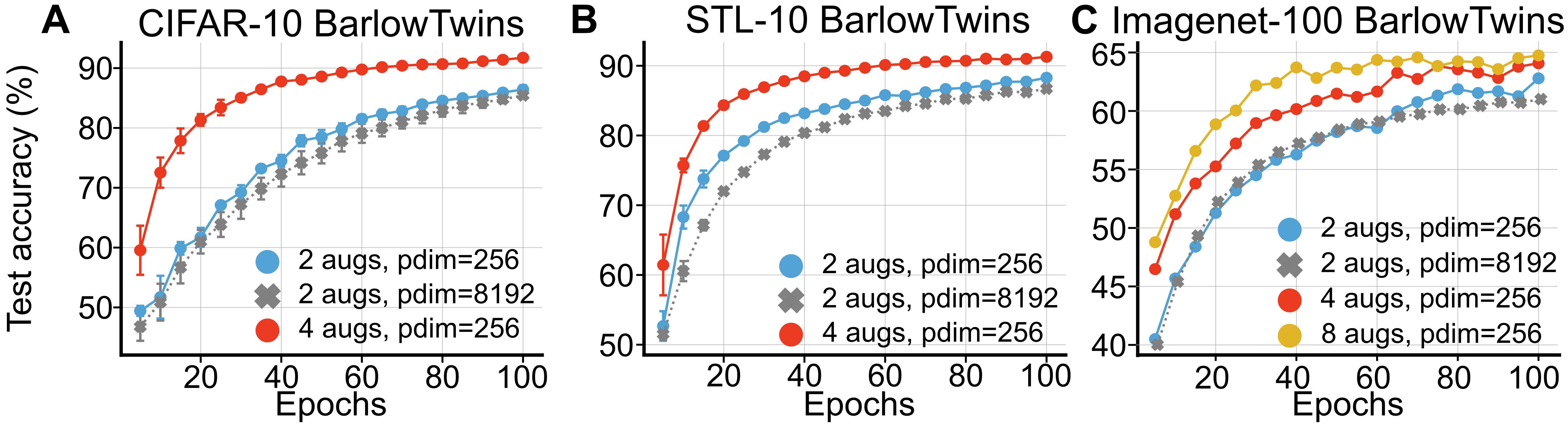}
    \vspace{-0.25cm}
    \caption{Using multiple augmentations improves representation learning performance and convergence. (A-C) Across BarlowTwins for CIFAR-10, STL-10 and Imagenet-100 pretraining, using 4 augmentations instead of 2 helps improve performance. Please see \Cref{sec:full_dset_pretraining} for more results.}
    \label{fig:multipatch}
\end{figure}
% \begin{figure*}[htbp]
%     \centering
%     \includegraphics[width=0.7\linewidth]{figs/FinalFigs/Fig4_multiPatch.png}
%     \vspace{-0.25cm}
%     \caption{Using multiple augmentations improves representation learning performance and convergence. (A-C) Across BarlowTwins and VICReg for CIFAR-10 and STL-10 pretraining, using 4 augmentations instead of 2 helps improve performance. (D-F) Although the 4-augmentations take longer for each epoch, its performance still trumps the 2-augmentation version of the algorithm at the same wall clock time.}
%     \label{fig:multipatch}
% \end{figure*}

\begin{table}[!ht]
 \small
  \centering
  \begin{tabular}{|c|c|c|c|}
    \hline
    augs & pdim & \textbf{BarlowTwins} & \textbf{VICReg} \\
     &  & Time (min) & Time (min) \\
    \hline
    \vspace{-2mm} \\
    2 & 8192 & $99.36 \pm 0.01$ & $94.36 \pm 0.01$ \\
    2 & 256 & $62.34 \pm 0.06$ & $51.73 \pm 0.04$ \\
    4 & 256 & \cellcolor{lightgray}{\bf 43.09 $\pm$ 0.20} & \cellcolor{lightgray}{\bf 39.02 $\pm$ 0.05} \\
    \hline
  \end{tabular}
  \caption{Using multiple augmentations yields faster convergence, with reduced time to reach baseline performance on CIFAR-10, i.e. performance of feature encoder pretrained with an 8192-dim projector and 2 augmentations.}
  \label{tab:fig4_table}
\end{table}

% To understand the role of multiple augmentations in pretraining in light of the augmentation-kernel interpretation, we propose \Cref{eq:ma_ncssl}, which generalizes Barlow-Twins and VICReg to the multi-augmentation setting. 
To demonstrate the role of multiple augmentations in pretraining, we adapt the invariance criterion of BarlowTwins/VICReg to be in line with \Cref{eq:ma_ncssl}.
In particular, for $\#augs \in \{2, 4, 8\}$, we pretrain a Resnet-50 encoder with our proposed loss. Building on the insight from the previous section, we use a 256-dimensional projector head for all multi-augmentation experiments. 

% In \Cref{fig:multipatch}, we track the downstream (linear readout) performance of the pretrained models across training epochs, i.e., we extract features from intermediate checkpoints and train a linear classifier on top of the features. 
In \Cref{fig:multipatch}, we track the downstream performance of the pretrained models across training epochs.
For performance evaluation, we use the linear evaluation protocol as outlined by \cite{chen2022intra}. 
\Cref{fig:multipatch}(A-C) shows that pretraining with multiple augmentations outperforms the 2-augmentation baseline. Furthermore, we observe that the four-augmentation pretrained models converge faster (both in terms of the number of epochs and wall-clock time) than their two-augmentation counterparts (see \Cref{fig:multipatch}(D-F)). Additionally, we show in \Cref{sec:swav_pretraining} that our framework can also be applied to multi-augmentation settings like SWaV, where not all augmentations are of the same resolution.

\begin{mdframed}
    {\bf Recommendatation:} Using multiple augmentations ( $>2$) is likely to improve convergence as well as downstream accuracy.
\end{mdframed}

%%ss.17.05.2024: Should we consider running some experiments with higher values of number of augs. We could present it on a curve with number of augs on the x-axis. My primary motivation is to make the above recommendation more concerete and add a caveat around when do we start seeing diminishing (or even negative) returns as we scale number of augmentations.

%%ss.17.05.2024: A more general comment - I think the recommendations in the paper are very useful but currently they are framed as if they are always valid. We should consider some experiments that show that the recommendations (like use more augmentations) are useful over a large scale (as we keep increasing the number of augmentations) but eventually leads to diminishing returns. 

\subsection{Sample Efficient Multi-augmentation Learning}

Data Augmentation can be viewed as a form of data inflation, where the number of training samples is increased by $k$ (for $k$ augmentations). In this section, we examine the role of multi-augmentation in improving sample efficiency. In particular, we are interested in understanding if the same performance can be achieved with a fraction of the pretraining dataset, simply by using more augmentations.

\begin{figure}[!ht]
    \centering
    \includegraphics[width=\linewidth]{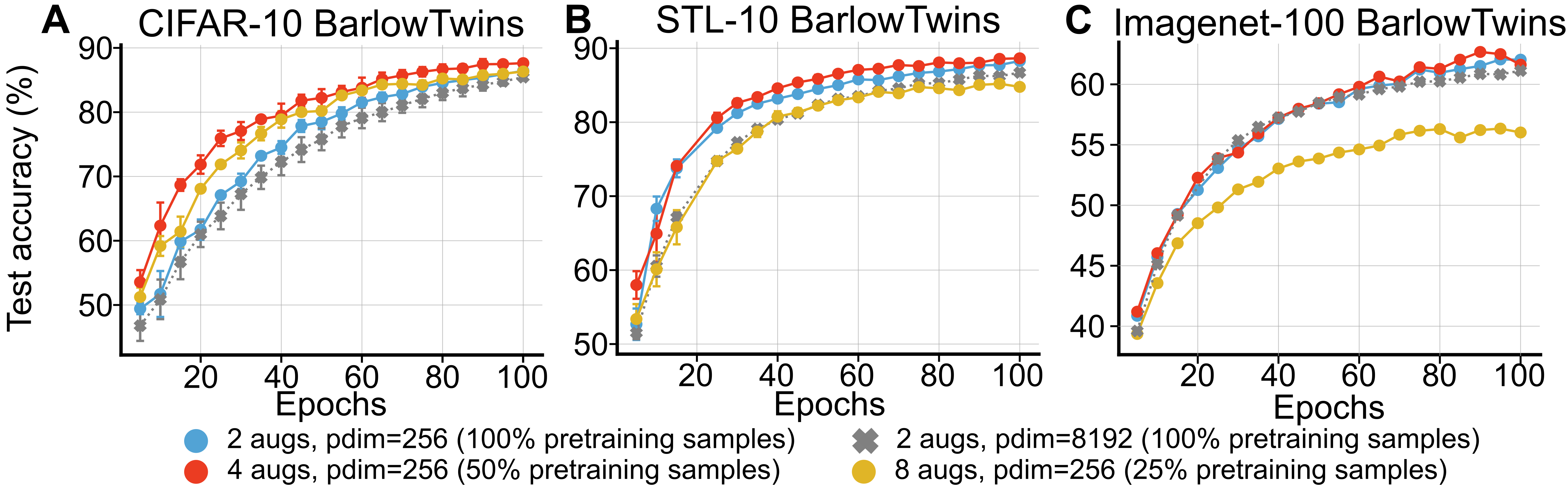}
    \caption{Multi-augmentation improves sample efficiency, recovering similar performance with significantly fewer unique samples in the pretraining dataset. Across BarlowTwins pretraining on CIFAR-10, STL-10 and Imagenet-100 for the same effective dataset size ($\#augs \times \#unique\_samples$), using more patches improves performance at the same epoch (A-C). However, a tradeoff exists wherein more data augmentations fail to improve performance in the scarce data regime.}
    \label{fig:multiPatch_fracTrain}
\end{figure}
% \begin{figure*}[htbp]
%     \centering
%     \includegraphics[width=0.7\linewidth]{figs/FinalFigs/Fig5_multiPatch_fracTrain.png}
%     \caption{Multi-augmentation improves sample efficiency, recovering similar performance with significantly fewer unique samples in the pretraining dataset. Across BarlowTwins and VICReg pretraining on CIFAR-10 and STL-10, for the same effective dataset size ($\#augs \times \#unique\_samples$), using more patches improves performance at the same epoch (A-C) or wall clock time (D-F). However, a tradeoff exists wherein doing more data augmentations fails to improve performance in the scarce data regime.}
%     \label{fig:multiPatch_fracTrain}
% \end{figure*}

\begin{table}[!ht]
 \small
  \centering
  \begin{tabular}{|c|c|c|c|c|}
    \hline
    augs & pdim & Percentage & \textbf{BarlowTwins} & \textbf{VICReg} \\
     &  & of Dataset & Time (min) & Time (min) \\
    \hline
    \vspace{-2mm} \\
    2 & 8192 & 100 \% & 63.43 $\pm$ 0.02 & 66.05 $\pm$ 0.01 \\
    2 & 256 & 100 \% & 39.52 $\pm$ 0.04 & 40.64 $\pm$ 0.04 \\
    4 & 256 & 50 \% & 28.25 $\pm$ 0.01 & \cellcolor{lightgray}{\bf32.39 $\pm$ 0.01} \\
    8 & 256 & 25 \% & \cellcolor{lightgray}{\bf27.74 $\pm$ 0.01} & 34.76 $\pm$ 0.01 \\
    \hline
  \end{tabular}
  \caption{Time required to pass 80\% accuracy on CIFAR-10 when pretraining on fraction of the dataset, while using multiple augmentations. See \Cref{fig:pareto} for further discussion.}
  \label{tab:fig5_table}
\end{table}

To examine the relation between the number of augmentations and sample efficiency, we fixed the effective size of the inflated dataset. This is achieved by varying the fraction of the unique samples in the pretraining dataset depending on the number of augmentations $k \in \{2, 4, 8\}$, e.g., we use 50\% of the dataset for 4 views. We then evaluate the performance of the pretrained models on the downstream task, where the linear classifier is trained on the same set of labeled samples. Strikingly, \Cref{fig:multiPatch_fracTrain} shows that using multiple augmentations can achieve similar (sometimes even better) performance with lesser pretraining samples, thereby indicating that more data augmentations can be used for feature learning to compensate for smaller pretraining datasets.

\begin{mdframed}
    {\bf Recommendation:}  In a low-data regime, using diverse \& multiple augmentations can be as effective as acquiring more unique samples.
\end{mdframed}

\section{Related Work}
 
{\bf Self-Supervised Pretraining} requires significant compute resources and most practitioners rely on empirical heuristics (see SSL cookbook \citep{balestriero2023cookbook} for a summary). 
% The SSL cookbook \citep{balestriero2023cookbook} provides a comprehensive summary of several widely adopted practices. 
While recent advances in SSL theory explore learning dynamics in linear (or shallow) models \citep{tian2021understanding,tian2020understanding}, with a focus on understanding dimensionality collapse \citep{ghosh2022investigating,jing2021understanding}, the theoretical underpinnings of most of the heuristics considered essential for good feature learning, are missing. 

\textbf{Contrastive SSL} has received more theoretical attention, owing to its connection with metric learning and noise contrastive estimation \citep{li2021self, DBLP:conf/nips/BalestrieroL22, DBLP:conf/iclr/0001HM23}. In particular, HaoChen \textit{et al.} \citep{haochen2021provable} provide a theoretical framework for the SimCLR loss from an augmentation graph perspective, which leads to practical recommendations. Subsequently, Garrido \textit{et al.} \citep{garrido2022duality} establish a duality between contrastive and non-contrastive learning objectives, further bridging the gap between theory and practice.
% Add 1-2 sentences about neural networks for learning eigenfunctions?

\textbf{Non-contrastive SSL} algorithms' theoretical foundations have received more attention recently \citep{cabannes2023ssl,zhai2023understanding}. Prior works \citep{agrawal2022alpha, garrido2022duality, garrido2023rankme} have demonstrated that with modified learning objectives, low-dimensional projectors yield representations with good downstream performance. Similarly, previous works have demonstrated notable performance boosts when using a multi-patch framework in contrastive \citep{dwibedi2021little} and non-contrastive SSL \citep{caron2018deep, yerxa2023learning}. However, the theoretical basis for the benefits and trade-offs of either low-dimensional projectors or multiple augmentations is largely unclear. It is worth noting that Schaeffer \textit{et al.} \citep{schaeffer2024towards} present an information-theoretic perspective of the recently proposed non-contrastive SSL loss that leverages multiple augmentations, namely MMCR \citep{yerxa2023learning}, but the computational advantages of using multiple augmentations on the learning dynamics is an active area of research.

{\bf Deep Learning theory} has made significant strides in understanding the optimization landscape and dynamics of supervised learning \citep{advani2020high}. In concurrent works \citep{zhai2023understanding, cabannes2023ssl}, the interplay between the inductive bias of data augmentations, architectures, and generalization has been explored from a purely theoretical perspective, establishing an equivalence among different SSL losses \citep{zhai2023understanding}. Furthermore, Simon \textit{et al.} \citep{simon2023stepwise} used a more straightforward formulation of the BarlowTwins loss and investigated the learning dynamics in linearized models for the case when the invariance and orthogonalization losses have equal penalties. Although such a setting rarely used in practice, their approach serves as an inspiration for our work in studying the learning dynamics of non-contrastive SSL losses.

\section{Discussion}

% {\bf Summary}: Our work presented a theoretical analysis that sheds light on the implicit bias of non-contrastive SSL algorithms. We used these insights to unravel the impact of crucial design heuristics and offer practical recommendations that improve convergence while maintaining accuracy. We also showed that the multi-augmentation framework can be used to learn good features from fewer unique samples in the pretraining dataset simply by improving the estimation of the data augmentation kernel.
{\bf Summary}: Our work builds on existing theoretical results that establish an equivalence among different SSL frameworks, and proposes a compute-efficient reformulation of the common SSL loss. Using this loss reformulation and a study of the optimization dynamics, we proposed practical recommendations to improve the sample and compute efficiency of SSL algorithms. Specifically, we recommended low-dimensional projectors with increased orthogonality constraints and multi-augmentation frameworks, and we verified the effectiveness of these recommendations empirically. It is worth noting that our multi-augmentation formulation improves the efficiency of learning without altering the desiderata of SSL, i.e. the network learns the same feature space using our proposed multi-augmentation framework as with the original SSL formulation in the limit of infinite pretraining budget. To demonstrate this equivalence between the original SSL loss and our proposed version, we show in \Cref{sec:longer_pretraining} that longer pretraining on the 2-augmentation loss leads to similar downstream performance as the multi-augmentation versions (4 and 8 augmentations).

We also showed that the multi-augmentation framework can be used to learn good features from fewer unique samples in the pretraining dataset simply by improving the estimation of the data augmentation kernel. This result has direct implications on improving the Pareto frontier of samples-vs-performance for SSL pretraining, wherein we can achieve better downstream performance when limited number of samples are available in the pretraining dataset.

\begin{figure}[htbp]
    \centering
    \begin{minipage}{0.48\linewidth}
        {\bf Pareto Optimal SSL}
       In the context of sample efficiency, training a model using two augmentations with different fractions of the dataset leads to a natural Pareto frontier, i.e., training on the full dataset achieves the best error but takes the most time ({\bf Baseline (2-Aug)}). Our extensive experiments demonstrate that using more than two augmentations improves the overall Pareto frontier, i.e., achieves better convergence while maintaining accuracy ({\bf Multi-Aug}). Strikingly, as shown in \Cref{fig:pareto}, we observe that we can either use a larger pretraining dataset or more augmentations for a target error level. Therefore, the number of augmentations can be used as a knob to control the sample efficiency of the pretraining routine.
    \end{minipage}\hfill
    \begin{minipage}{0.48\linewidth}
        \centering
        \includegraphics[width=\linewidth]{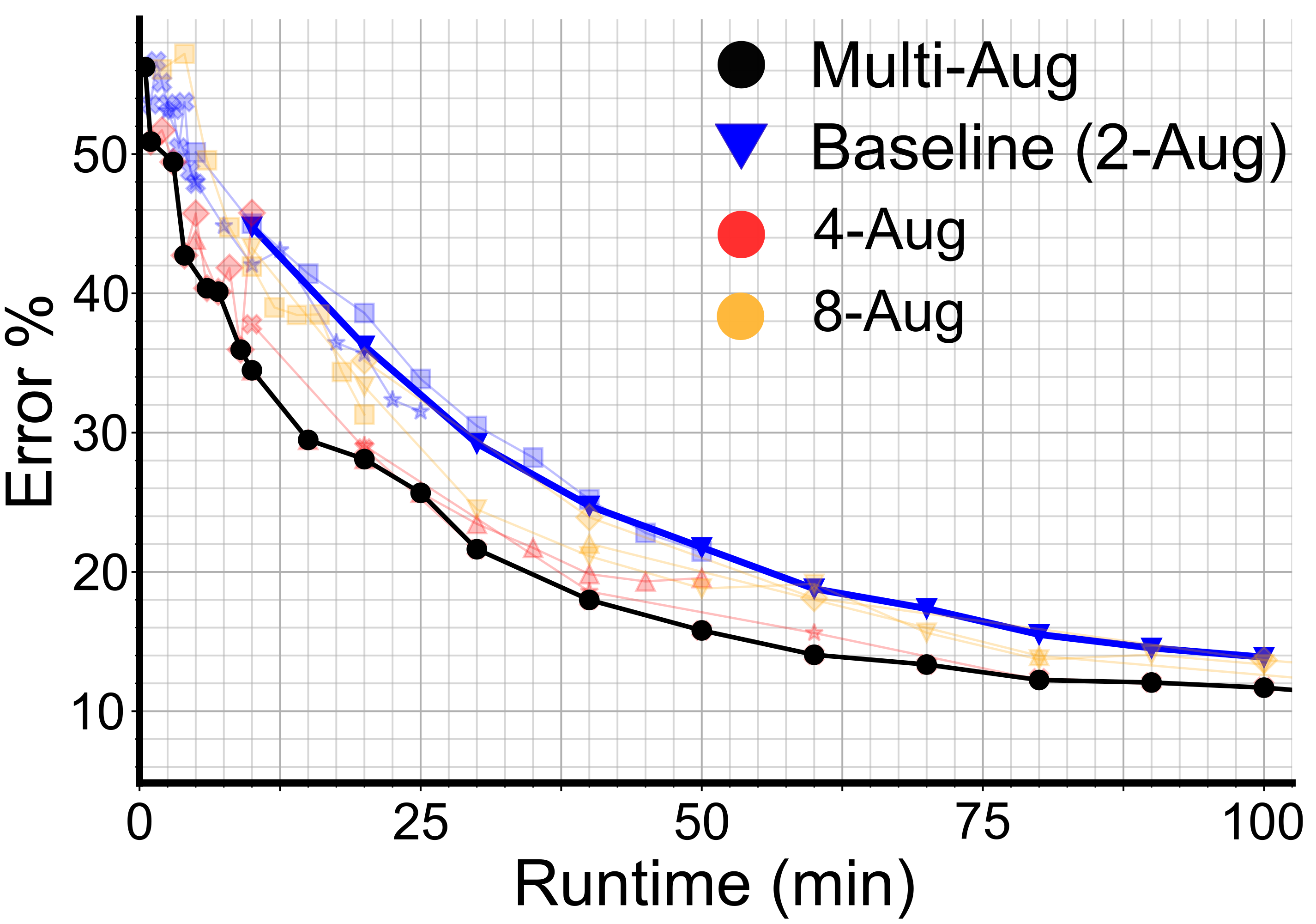}%{h.PNG} % second figure itself
        \caption{Using $>2$ augmentations with a fraction of the dataset improves overall Pareto frontier, speeding runtime up to $\sim2\times$.}
        \label{fig:pareto}
    \end{minipage}
\end{figure}

{\bf Connections to Downstream performance}: While our core theoretical results are aimed at accelerating convergence of the SSL loss itself, our empirical results highlight an improved downstream task performance earlier during pretraining. While this discrepancy might seem counter-intuitive at first, it is worth noting that the SSL loss inherent influences downstream performance as it encourages clustering of semantically similar images in the representation space. Such clustering properties in the representation space facilitates easier classification through methods k-nearest neighbors or linear decoding for a large number of tasks that rely on the semantic content of images. Previous works \cite{ghosh2022investigating, agrawal2022alpha, garrido2023rankme, thilak2024lidar} have discussed in detail how certain geometric properties of the learned representation space are connected to the linear classification performance for arbitrary decision boundaries, in expectation. However, an in-depth analysis of downstream tasks that are more amenable to linear decoding from the learned SSL representation space requires framing metrics of alignment between the pretraining objective (SSL desiderate) and the downstream task labels, and is an active area of research.

{\bf Open Questions}: Looking ahead, it would be exciting to extend this analysis to other categories of SSL algorithms, such as Masked AutoEncoders (MAE). Furthermore, our insights provide opportunities to explore sample-efficient methods that rely on less data, which is particularly important in critical domains such as medical imaging, where data is often relatively scarce and expensive.
On a different note, it is intriguing that animals often spend extended periods of time exploring novel objects, likely to gain multiple views of the object \citep{berlyne1950novelty, leger2013object}. Given the theoretical underpinnings of the computational benefits of multi-augmentation SSL outlined in our work, it would be exciting to develop models of biological learning that leverage these insights and enable sample-efficient continual learning in similar environments.

{\bf Limitations}: Our algorithm relies on multiple augmentations of the same image to improve the estimation of the data-augmentation kernel. Though this approach speeds up the learning process, it also adds some extra computational overhead, which means that the impact of faster learning on wall-clock time is less than might be hoped for. One way to mitigate the effects of this limitation would be to scale up to a multi-GPU setting, since the computations for each augmentation can be run on a separate GPU in parallel. This could help ensure that the improved speed of learning directly translates to a significantly reduced wall-clock time for training.

{\bf Impact Statement}: The goal of our work is to advance the general field of visual representation learning. Although there are potential downstream societal consequences of our work, we feel there are no direct consequences that must be specifically highlighted here.
\section*{Acknowledgements}
The authors would like to thank Arnab Kumar Mondal for insightful discussions that helped shape the project’s scope, Colleen Gillon for aesthetic contribution to the figures and Florian Bordes for helping setup the \href{https://github.com/facebookresearch/FFCV-SSL}{FFCV-SSL} library. The authors are also grateful to Aidan Sirbu, Chen Sun, Jonathan Cornford, Roman Pogodin, Zahraa Chorghay and the anonymous reviewers whose comments and suggestions have significantly enhanced the quality and presentation of the results. This research was generously supported by Vanier Canada Graduate Scholarship (A.G.); NSERC (Discovery Grant: RGPIN-2020-05105; RGPIN-2018-04821; Discovery Accelerator Supplement: RGPAS-2020-00031; Arthur B. McDonald
Fellowship: 566355-2022), Healthy Brains, Healthy Lives (New Investigator Award: 2b-NISU-8), and CIFAR Learning in Machines \& Brains Program (B.A.R.); Canada CIFAR AI Chair program (A.O. \& B.A.R.). The authors also acknowledge the material support of NVIDIA in the form of computational resources, as well as the compute resources, software and technical help provided by \href{https://mila.quebec/en}{Mila} (mila.quebec).

\bibliography{references}
\bibliographystyle{plain}
\newpage

% %%%%%%%%%%%%%%%%%%%%%%%%%%%%%%%%%%%%%%%%%%%%%%%%%%%%%%%%%%%%%%%%%%%%%%%%%%%%%%%
% %%%%%%%%%%%%%%%%%%%%%%%%%%%%%%%%%%%%%%%%%%%%%%%%%%%%%%%%%%%%%%%%%%%%%%%%%%%%%%%
% % APPENDIX
% %%%%%%%%%%%%%%%%%%%%%%%%%%%%%%%%%%%%%%%%%%%%%%%%%%%%%%%%%%%%%%%%%%%%%%%%%%%%%%%
% %%%%%%%%%%%%%%%%%%%%%%%%%%%%%%%%%%%%%%%%%%%%%%%%%%%%%%%%%%%%%%%%%%%%%%%%%%%%%%%

\clearpage
\newpage
\section*{NeurIPS Paper Checklist}

\begin{enumerate}

\item {\bf Claims}
    \item[] Question: Do the main claims made in the abstract and introduction accurately reflect the paper's contributions and scope?
    \item[] Answer: \answerYes{} % Replace by \answerYes{}, \answerNo{}, or \answerNA{}.
    \item[] Justification: Our main claims are backed by theoretical and empirical results. \Cref{theorem:eigen} presents our functionally-equivalent compute-efficient formulation of the SSL objective, and \Cref{theorem:dynamics} demonstrates the implicit bias of gradient descent during optimizing the SSL loss. Our empirical results demonstrate the utility of our theoretical insights in improving the parameter overhead of good feature learning, optimization convergence and the sample efficiency.   
    %\justificationTODO{}
    % \item[] Guidelines:
    % \begin{itemize}
    %     \item The answer NA means that the abstract and introduction do not include the claims made in the paper.
    %     \item The abstract and/or introduction should clearly state the claims made, including the contributions made in the paper and important assumptions and limitations. A No or NA answer to this question will not be perceived well by the reviewers. 
    %     \item The claims made should match theoretical and experimental results, and reflect how much the results can be expected to generalize to other settings. 
    %     \item It is fine to include aspirational goals as motivation as long as it is clear that these goals are not attained by the paper. 
    % \end{itemize}

\item {\bf Limitations}
    \item[] Question: Does the paper discuss the limitations of the work performed by the authors?
    \item[] Answer: \answerYes{} % Replace by \answerYes{}, \answerNo{}, or \answerNA{}.
    \item[] Justification: We have a section on limitations in the discussion. 

\item {\bf Theory Assumptions and Proofs}
    \item[] Question: For each theoretical result, does the paper provide the full set of assumptions and a complete (and correct) proof?
    \item[] Answer: \answerYes{} % Replace by \answerYes{}, \answerNo{}, or \answerNA{}.
    \item[] Justification: The formal statements alongside proofs are presented in the supplementary material (\Cref{sec:theory_general,sec:kernel_SSL_theory,sec:multi_augs_theory}).
    % \justificationTODO{}
    % \item[] Guidelines:
    % \begin{itemize}
    %     \item The answer NA means that the paper does not include theoretical results. 
    %     \item All the theorems, formulas, and proofs in the paper should be numbered and cross-referenced.
    %     \item All assumptions should be clearly stated or referenced in the statement of any theorems.
    %     \item The proofs can either appear in the main paper or the supplemental material, but if they appear in the supplemental material, the authors are encouraged to provide a short proof sketch to provide intuition. 
    %     \item Inversely, any informal proof provided in the core of the paper should be complemented by formal proofs provided in appendix or supplemental material.
    %     \item Theorems and Lemmas that the proof relies upon should be properly referenced. 
    % \end{itemize}

    \item {\bf Experimental Result Reproducibility}
    \item[] Question: Does the paper fully disclose all the information needed to reproduce the main experimental results of the paper to the extent that it affects the main claims and/or conclusions of the paper (regardless of whether the code and data are provided or not)?
    \item[] Answer: \answerYes{} % Replace by \answerYes{}, \answerNo{}, or \answerNA{}.
    \item[] Justification: We provide implementation details in the supplementary material (\Cref{sec:implementation}). We have also released our code base in the public github repo, \href{https://github.com/kumarkrishna/fastssl}{FastSSL}, to facilitate the implementation of our proposed framework.

\item {\bf Open access to data and code}
    \item[] Question: Does the paper provide open access to the data and code, with sufficient instructions to faithfully reproduce the main experimental results, as described in supplemental material?
    \item[] Answer: \answerYes{}. % Replace by \answerYes{}, \answerNo{}, or \answerNA{}.
    \item[] Justification: We use open-access datasets, like CIFAR, STL and Imagenet. Our code base can be found in the public github repo, \href{https://github.com/kumarkrishna/fastssl}{FastSSL}
    % \justificationTODO{}
    \item[] Guidelines:
    % \begin{itemize}
    %     \item The answer NA means that paper does not include experiments requiring code.
    %     \item Please see the NeurIPS code and data submission guidelines (\url{https://nips.cc/public/guides/CodeSubmissionPolicy}) for more details.
    %     \item While we encourage the release of code and data, we understand that this might not be possible, so “No” is an acceptable answer. Papers cannot be rejected simply for not including code, unless this is central to the contribution (e.g., for a new open-source benchmark).
    %     \item The instructions should contain the exact command and environment needed to run to reproduce the results. See the NeurIPS code and data submission guidelines (\url{https://nips.cc/public/guides/CodeSubmissionPolicy}) for more details.
    %     \item The authors should provide instructions on data access and preparation, including how to access the raw data, preprocessed data, intermediate data, and generated data, etc.
    %     \item The authors should provide scripts to reproduce all experimental results for the new proposed method and baselines. If only a subset of experiments are reproducible, they should state which ones are omitted from the script and why.
    %     \item At submission time, to preserve anonymity, the authors should release anonymized versions (if applicable).
    %     \item Providing as much information as possible in supplemental material (appended to the paper) is recommended, but including URLs to data and code is permitted.
    % \end{itemize}

\item {\bf Experimental Setting/Details}
    \item[] Question: Does the paper specify all the training and test details (e.g., data splits, hyperparameters, how they were chosen, type of optimizer, etc.) necessary to understand the results?
    \item[] Answer: \answerYes{} % Replace by \answerYes{}, \answerNo{}, or \answerNA{}.
    \item[] Justification: We present details of the experiment setup and results in Section 4 of the main paper, and additional implementation details in Appendix D.
    %\justificationTODO{}
    % \item[] Guidelines:
    % \begin{itemize}
    %     \item The answer NA means that the paper does not include experiments.
    %     \item The experimental setting should be presented in the core of the paper to a level of detail that is necessary to appreciate the results and make sense of them.
    %     \item The full details can be provided either with the code, in appendix, or as supplemental material.
    % \end{itemize}

\item {\bf Experiment Statistical Significance}
    \item[] Question: Does the paper report error bars suitably and correctly defined or other appropriate information about the statistical significance of the experiments?
    \item[] Answer: \answerYes % Replace by \answerYes{}, \answerNo{}, or \answerNA{}.
    \item[] Justification: We report standard error bars, computed over 3 seeds, for all result plots and tables. 
    % \justificationTODO{}
    % \item[] Guidelines:
    % \begin{itemize}
    %     \item The answer NA means that the paper does not include experiments.
    %     \item The authors should answer "Yes" if the results are accompanied by error bars, confidence intervals, or statistical significance tests, at least for the experiments that support the main claims of the paper.
    %     \item The factors of variability that the error bars are capturing should be clearly stated (for example, train/test split, initialization, random drawing of some parameter, or overall run with given experimental conditions).
    %     \item The method for calculating the error bars should be explained (closed form formula, call to a library function, bootstrap, etc.)
    %     \item The assumptions made should be given (e.g., Normally distributed errors).
    %     \item It should be clear whether the error bar is the standard deviation or the standard error of the mean.
    %     \item It is OK to report 1-sigma error bars, but one should state it. The authors should preferably report a 2-sigma error bar than state that they have a 96\% CI, if the hypothesis of Normality of errors is not verified.
    %     \item For asymmetric distributions, the authors should be careful not to show in tables or figures symmetric error bars that would yield results that are out of range (e.g. negative error rates).
    %     \item If error bars are reported in tables or plots, The authors should explain in the text how they were calculated and reference the corresponding figures or tables in the text.
    % \end{itemize}

\item {\bf Experiments Compute Resources}
    \item[] Question: For each experiment, does the paper provide sufficient information on the computer resources (type of compute workers, memory, time of execution) needed to reproduce the experiments?
    \item[] Answer: \answerYes{} % Replace by \answerYes{}, \answerNo{}, or \answerNA{}.
    \item[] Justification: All our CIFAR and STL experiments were done on a single 48-GB RTX8000 GPU and all Imagenet experiments were performed on 2 40-GB A100 GPUs. All experiments were performed on the Mila cluster, aided by compute resources, software and technical help provided by Mila (mila.quebec).
    % \justificationTODO{}
    % \item[] Guidelines:
    % \begin{itemize}
    %     \item The answer NA means that the paper does not include experiments.
    %     \item The paper should indicate the type of compute workers CPU or GPU, internal cluster, or cloud provider, including relevant memory and storage.
    %     \item The paper should provide the amount of compute required for each of the individual experimental runs as well as estimate the total compute. 
    %     \item The paper should disclose whether the full research project required more compute than the experiments reported in the paper (e.g., preliminary or failed experiments that didn't make it into the paper). 
    % \end{itemize}
    
\item {\bf Code Of Ethics}
    \item[] Question: Does the research conducted in the paper conform, in every respect, with the NeurIPS Code of Ethics \url{https://neurips.cc/public/EthicsGuidelines}?
    \item[] Answer: \answerYes{} % Replace by \answerYes{}, \answerNo{}, or \answerNA{}.
    \item[] Justification: We adhere to the NeurIPS Code of Ethics in our work, and research in general.
    % \justificationTODO{}
    % \item[] Guidelines:
    % \begin{itemize}
    %     \item The answer NA means that the authors have not reviewed the NeurIPS Code of Ethics.
    %     \item If the authors answer No, they should explain the special circumstances that require a deviation from the Code of Ethics.
    %     \item The authors should make sure to preserve anonymity (e.g., if there is a special consideration due to laws or regulations in their jurisdiction).
    % \end{itemize}

\item {\bf Broader Impacts}
    \item[] Question: Does the paper discuss both potential positive societal impacts and negative societal impacts of the work performed?
    \item[] Answer: \answerYes{} % Replace by \answerYes{}, \answerNo{}, or \answerNA{}.
    \item[] Justification: We add a statement on societal impact in the Discussion section. Since the goal of our work is to advance the general field of visual representation learning, we feel there are no direct consequences that must be specifically highlighted here. Although we recognize that there might be potential downstream consequences that warrant attention while building intelligent systems that leverage this work.

\item {\bf Safeguards}
    \item[] Question: Does the paper describe safeguards that have been put in place for responsible release of data or models that have a high risk for misuse (e.g., pretrained language models, image generators, or scraped datasets)?
    \item[] Answer: \answerNA{} % Replace by \answerYes{}, \answerNo{}, or \answerNA{}.
    \item[] Justification: We do not release any new data or state-of-the-art models.
    % \justificationTODO{}
    % \item[] Guidelines:
    % \begin{itemize}
    %     \item The answer NA means that the paper poses no such risks.
    %     \item Released models that have a high risk for misuse or dual-use should be released with necessary safeguards to allow for controlled use of the model, for example by requiring that users adhere to usage guidelines or restrictions to access the model or implementing safety filters. 
    %     \item Datasets that have been scraped from the Internet could pose safety risks. The authors should describe how they avoided releasing unsafe images.
    %     \item We recognize that providing effective safeguards is challenging, and many papers do not require this, but we encourage authors to take this into account and make a best faith effort.
    % \end{itemize}

\item {\bf Licenses for existing assets}
    \item[] Question: Are the creators or original owners of assets (e.g., code, data, models), used in the paper, properly credited and are the license and terms of use explicitly mentioned and properly respected?
    \item[] Answer: \answerYes{} % Replace by \answerYes{}, \answerNo{}, or \answerNA{}.
    \item[] Justification: We acknowledge and cite the datasets and model architectures used in this work. Our codebase is publicly available on our github repo, \href{https://github.com/kumarkrishna/fastssl}{FastSSL}. Moreover, our codebase relies on the Python packages of \href{https://pytorch.org/}{PyTorch}, \href{https://github.com/libffcv/ffcv}{FFCV} \citep{leclerc2022ffcv} and \href{https://github.com/facebookresearch/FFCV-SSL}{FFCV-SSL} \citep{bordes2023ffcv_ssl}, which are referred to in the github repo README.  
    % \justificationTODO{}
    % \item[] Guidelines:
    % \begin{itemize}
    %     \item The answer NA means that the paper does not use existing assets.
    %     \item The authors should cite the original paper that produced the code package or dataset.
    %     \item The authors should state which version of the asset is used and, if possible, include a URL.
    %     \item The name of the license (e.g., CC-BY 4.0) should be included for each asset.
    %     \item For scraped data from a particular source (e.g., website), the copyright and terms of service of that source should be provided.
    %     \item If assets are released, the license, copyright information, and terms of use in the package should be provided. For popular datasets, \url{paperswithcode.com/datasets} has curated licenses for some datasets. Their licensing guide can help determine the license of a dataset.
    %     \item For existing datasets that are re-packaged, both the original license and the license of the derived asset (if it has changed) should be provided.
    %     \item If this information is not available online, the authors are encouraged to reach out to the asset's creators.
    % \end{itemize}

\item {\bf New Assets}
    \item[] Question: Are new assets introduced in the paper well documented and is the documentation provided alongside the assets?
    \item[] Answer: \answerNA{} % Replace by \answerYes{}, \answerNo{}, or \answerNA{}.
    \item[] Justification: No new assets are released.
    % \justificationTODO{}
    % \item[] Guidelines:
    % \begin{itemize}
    %     \item The answer NA means that the paper does not release new assets.
    %     \item Researchers should communicate the details of the dataset/code/model as part of their submissions via structured templates. This includes details about training, license, limitations, etc. 
    %     \item The paper should discuss whether and how consent was obtained from people whose asset is used.
    %     \item At submission time, remember to anonymize your assets (if applicable). You can either create an anonymized URL or include an anonymized zip file.
    % \end{itemize}

\item {\bf Crowdsourcing and Research with Human Subjects}
    \item[] Question: For crowdsourcing experiments and research with human subjects, does the paper include the full text of instructions given to participants and screenshots, if applicable, as well as details about compensation (if any)? 
    \item[] Answer: \answerNA{} % Replace by \answerYes{}, \answerNo{}, or \answerNA{}.
    \item[] Justification: Our work does not involve crowdsourcing nor research with human subjects.
    % \justificationTODO{}
    % \item[] Guidelines:
    % \begin{itemize}
    %     \item The answer NA means that the paper does not involve crowdsourcing nor research with human subjects.
    %     \item Including this information in the supplemental material is fine, but if the main contribution of the paper involves human subjects, then as much detail as possible should be included in the main paper. 
    %     \item According to the NeurIPS Code of Ethics, workers involved in data collection, curation, or other labor should be paid at least the minimum wage in the country of the data collector. 
    % \end{itemize}

\item {\bf Institutional Review Board (IRB) Approvals or Equivalent for Research with Human Subjects}
    \item[] Question: Does the paper describe potential risks incurred by study participants, whether such risks were disclosed to the subjects, and whether Institutional Review Board (IRB) approvals (or an equivalent approval/review based on the requirements of your country or institution) were obtained?
    \item[] Answer: \answerNA{} % Replace by \answerYes{}, \answerNo{}, or \answerNA{}.
    \item[] Justification: Our work does not involve research with human subjects.
    % \justificationTODO{}
    % \item[] Guidelines:
    % \begin{itemize}
    %     \item The answer NA means that the paper does not involve crowdsourcing nor research with human subjects.
    %     \item Depending on the country in which research is conducted, IRB approval (or equivalent) may be required for any human subjects research. If you obtained IRB approval, you should clearly state this in the paper. 
    %     \item We recognize that the procedures for this may vary significantly between institutions and locations, and we expect authors to adhere to the NeurIPS Code of Ethics and the guidelines for their institution. 
    %     \item For initial submissions, do not include any information that would break anonymity (if applicable), such as the institution conducting the review.
    % \end{itemize}

\end{enumerate}

\newpage
\appendix
\section{Hilbert Space of functions}
\label{sec:theory_general}
\subsection{Functions and inner product space}

\begin{definition}
Given $\XX,\rhox$, and $f,g : \XX \to \R$, define the $\Lt$ inner product and norm, respectively,  
\begin{equation}\label{L2innerprod}
	(f,g)_{\Ltr} = \int f(x) g(x)d\rhox(x), \quad \|f\|^2_\Ltr = (f,f)_\Ltr
\end{equation}
Define 
$$
L^2(\rho,X) = \left\{ f: \XX \to \R \ \mid \|f\|^2_\Ltr < \infty   \right \}
$$
to be the (equivalence class) of functions with finite $\Ltr$ norm. 
\end{definition}

\subsection{Spectral theory}
In this section we quote the relevant (abstract) Hilbert Space theory.

\begin{definition}[Spectral Operator]
Given orthogonal functions, $\Phi = (\phi_i)_{i\in I}$ in $\Lt$,  
and non-negative $\Lambda = (\lambda_i)_{i\in I}$, with $\|\Lambda||_2^2  = \sum_{i\in I} \lambda_i^2 < \infty$.
Call $(\Phi, \Lambda)$ a spectral pair and define the corresponding spectral operator by
\begin{equation}\label{Spectral}
	\Op_{\Phi,\Lambda}(h)=\sum_{j=1}^{\infty} \lambda_j\left ( h, \phi_j\right) \phi_j, 
\end{equation}
\end{definition}

\begin{theorem}[Spectral Decomposition]
Suppose $H$ is a Hilbert space.
A symmetric positive-definite Hilbert-Schmidt operator $\Op:\Hh \to \Hh$ admits the spectral decomposition \eqref{Spectral}
%\begin{equation}
%		\Op(h) =\sum_{j=1}^{\infty} \lambda_j\left ( h, \phi_j\right) \phi_j, 
%\end{equation}
with orthonormal $\phi_j$ which are the eigenfunctions of $\Op$, i.e. $\Op\left(\phi_j\right)=\lambda_j \phi_j$. The $\phi_j$ can be extended to a basis by adding a complete orthonormal system in the orthogonal complement of the subspace spanned by the original $\phi_j$. 	
\end{theorem}
\begin{remark}
The $\phi_j$ in \eqref{Spectral} can thus be assumed to form a basis, but some $\lambda_j$ may be zero.	
\end{remark}

%\section{Hilbert Space Operator Theory}\label{sec:HSTheory}
From~\cite{horvath2012inference}. Theorem proved in \cite{gohberg1990classes}.
Denote by $\mathcal{L}$ the space of bounded (continuous) linear operators on $\Hh$ with the norm
$$
\|\Op\|_{\mathcal{L}}=\sup \{\|\Op(x)\|  \mid \|x\| \leq 1\} .
$$

\begin{definition}[Compact Operators]
An operator $\Op \in \mathcal{L}$ is said to be compact if there exist two orthonormal bases $\left\{g_j\right\}$ and $\left\{f_j\right\}$, and a real sequence $\left\{\lambda_j\right\}$ converging to zero, such that
\begin{equation}\label{Compact}\tag{Compact}
	\Op(h)=\sum_{j=1}^{\infty} \lambda_j (h,g_j)f_j , \quad h \in \Hh,
\end{equation}
The $\lambda_j$ may be assumed positive.
The existence of representation \eqref{Compact} is equivalent to the condition: $\Op$ maps every bounded set into a compact set. Compact operators are also called completely continuous operators. Representation \eqref{Compact} is called the singular value decomposition.	
\end{definition}

\begin{definition}[Hilbert-Schmidt Operators]
A compact operator admitting representation \eqref{Compact} is said to be a Hilbert-Schmidt operator if $\sum_{j=1}^{\infty} \lambda_j^2<\infty$.
The space $\mathcal{S}$ of Hilbert-Schmidt operators is a separable Hilbert space with the scalar product
\begin{equation}\label{eq1}
	\left\langle\Op_1, \Op_2\right\rangle_{\mathcal{S}}
	=\sum_{i=1}^{\infty}\left (\Op_1\left(f_i\right), \Op_2\left(f_i\right)\right),
\end{equation}
where $\left\{f_i\right\}$ is an arbitrary orthonormal basis. Note  the value of \eqref{eq1} is independent of the basis.  
The corresponding norm is 
\begin{equation}
	\label{HScond}\tag{HS}
	\|\Op\|_{\mathcal{S}}^2=\sum_{j \geq 1} \lambda_j^2
\end{equation}
\end{definition}

One can show that 
$$
\|\Op\|_{\mathcal{L}} \leq\|\Op\|_{\mathcal{S}}
$$

\begin{definition}
An operator $\Op \in \mathcal{L}$ is said to be symmetric if
$$
\langle\Op(f), g\rangle=\langle f, \Op(g)\rangle, \quad f,g \in \Hh,
$$
and positive-definite if
$$
\langle\Op(f), f\rangle \geq 0, \quad f \in \Hh .
$$
(An operator with the last property is sometimes called positive semidefinite, and the term positive-definite is used when the inequality is strict.)	
\end{definition}

% \begin{document} 
% \title{SSL Functional Features}
% \author{Adam M. Oberman}
% \date{{Compiled on \today\ }}
%

% \maketitle
%\tableofcontents 

% \section{Features from Data Augmentation Mappings}
\section{Data augmentation kernel perspective of non-contrastive SSL}
\label{sec:kernel_SSL_theory}

\begin{theorem}
Let $G(x)$ be the infinite Mercer features of the backward data augmentation covariance kernels, $k^{DAB}$. Let $F(x) = (f_1(x), f_2(x), \dots, f_k(x))$ be the features given by minimizing the following data augmentation invariance loss
\begin{align}
L(F) = 	\sum_{i=1}^{N_k} \|T_M f_i - f_i\|^2_\Lt ,\quad \text{subject to }\quad (f_i, f_j)_\Ltr = \delta_{ij}
\end{align}
which includes the orthogonality constraint. Then, $V(F) \subset V(G)$ ,  $V(F) \to V(G)$ as $N_k\to\infty$. 
%Moreover, the same result holds is we use $k^{DAF}$ instead of $k^{DAB}$.  
\label{theorem:eigen_appendix}
\end{theorem}

The idea of the proof uses the fact that, as linear operators, $T_{k^{DAB}} = T_M^\top T_M$ and that $T_{k^{DAF}} = T_MT_M^\top$.  Then we use spectral theory of compact operators, which is analogue of the Singular Value Decomposition in Hilbert Space, to show that eigenfunctions of $T_M^\top T_M$ operator are the same as those obtained from optimizing $L(F)$.  A similar result can be obtained using $k^{DAF}$ and $T_M^\top$.

% However, note that since the eigenvalues of the operators are not the same, truncating to finite dimensions may lead to a different choice of features. 

Note that $L(F)$ is the constrained optimization formulation of the BarlowTwins loss. 
% Furthermore, VICReg loss can also be seen as an unconstrained optimization formulation of $L(F)$ with the additional constraint that $(f_i,f_i) \geq \gamma$ $\forall i \in \{1, 2 \dots N_k\}$. 
Furthermore, $L(F)$ with the additional constraint that $(f_i,f_i) \geq \gamma$ $\forall i \in \{1, 2 \dots N_k\}$ is the constrained optimization formulation of the VICReg loss.

\subsection{Proof of theorem 3.1}
We show we can factor the linear operator, leading to a practical algorithm. Here, we show that we can capture the backward data augmentation kernel with the forward data augmentation averaging operator
% \emph{TODO: $T_{K}^F = T_M T_M^\top$ }

\begin{lemma}
Using the definitions above, and with $k$ in \eqref{Tkdefn} given by  $k^{DAB}$, 
\[
T_k = T_M^\top T_M
\]
\end{lemma}

\begin{proof}
First, define the non-negative definite bilinear form 
\begin{equation}	
\label{BilFormDefn}
B^{VAR}(f,g) 	= (\Op_M f, \Op_M g)_\Ltr  
\end{equation}
Given the backwards data augmentation covariance  kernel, $k^{DAB}$, define
\[
B^{DAB}(f,g) = (T_k f, g)_\Ltr 
\]
We claim, that
\begin{equation}
	B^{VAR} = B^{DA,B}
\end{equation}
This follows from the following calculation,
	% \begin{align}
	% 	B^{DA,B}(f,g) &= (T_k f, g)_\Ltr \\
	% 	&= \EE_x [T_k f(x), g(x)] = \EE_x \EE_z [k_{DA,B}(z,x)f(z) g(x) ]\\
	% 	& = \EE_x \EE_z \EE_{x_0} [ p(x \mid x_0) p(z \mid x_0 )f(z) g(x) ]\\
	% 	&=  \EE_{x_0}\left [   \EE_x  [p(x\mid x_0)  g(x)], \EE_z [p(z \mid x_0 )f(z)] , \right ] = \EE_{x_0} \Op_M f(x_0) \Op_M g(x_0)  \\
	% 	&= (\Op_M f, \Op_M g)_\Ltr  = B^{VAR}(f,g) 
	% \end{align}
        \begin{align}
		B^{DA,B}(f,g) &= (T_k f, g)_\Ltr \\
		&= \EE_x [T_k f(x), g(x)] = \EE_x \EE_z [k_{DA,B}(z,x)f(z) g(x) ]\\
		& = \EE_x \EE_z \EE_{x_0} \left[ \frac{p(x_0 \mid x)}{\rho(x_0)} \frac{p(x_0 \mid z)}{\rho(x_0)}f(z) g(x) \right]\\
            &= \EE_{x_0} \left[ \sum_x \left(\frac{\rho(x)p(x_0 \mid x)}{\rho(x_0)} g(x)\right) \sum_z \left(\frac{\rho(z)p(x_0 \mid z)}{\rho(x_0)} f(z)\right) \right] \\
            &= \EE_{x_0} \left[ \sum_x \left(p(x \mid x_0) g(x) \right) \sum_z \left(p(z \mid x_0) f(z) \right)\right] \quad \quad \text{[Using Bayes' rule]} \\
            &= \EE_{x_0} \left[ \Op_M f(x_0) \Op_M g(x_0)\right] = (\Op_M f, \Op_M g)_\Ltr  = B^{VAR}(f,g) 
		% &=  \EE_{x_0}\left [   \EE_x  [p(x\mid x_0)  g(x)], \EE_z [p(z \mid x_0 )f(z)] , \right ] = \EE_{x_0} \Op_M f(x_0) \Op_M g(x_0)  \\
		% &= (\Op_M f, \Op_M g)_\Ltr  = B^{VAR}(f,g) 
	\end{align}
\end{proof}

For implementations, it is more natural to consider \emph{invariance} to data augmentations. 
\begin{theorem}[equivalent eigenfunctions]
\label{thm.c3}
Assume that $\Op_M$ is a compact operator.  Define the  invariance bilinear form 
\begin{equation}\label{Binv}
	B^{INV}(f,g) = ( T_M f - f, T_M g - g)
\end{equation}
Then $B^{INV}$, $B^{VAR}$ share the same set of eigenfunctions.  Moreover, these are the same as the eigenfunctions of $B^{DA,B}$.
In particular, for any eigenfunction $f_j$ of $B^{VAR}$,   with eigenvalue $\lambda_j$, then $f_j$ is also and eigenfunction  of $B^{INV}$, with the corresponding eigenvalue given by $(\sqrt{\lambda_j} - 1)^2$.
\end{theorem}
\begin{proof}

Define $\Op_{MM}$ by,
\begin{equation}\label{AdjointFormulation}
\Op_{MM}f = \Op_M^\top  \Op_M f
\end{equation}
Define 
\begin{equation}
	\Op_{MS} = (T_M -  I)^\top(T_M-I)
\end{equation}
Note, by the assumption of compactness, $\Op_M$ has the Singular Value Decomposition, (see the Hilbert Space section for \eqref{SVD}), 
\begin{equation}\label{SVD}\tag{SVD}
	\Op_M(h)=\sum_{j=1}^{\infty} \lambda_j (h,g_j)f_j
\end{equation}
Let $f_j$ be any right eigenvector of $T_M$, with eigenvalue $\mu_j$.  Then $f_j$ is also a right eigenvector $T_M-I$, with eigenvalue $\mu_j - 1$.
So we see that $T_{MM}$ has $f_j$ as an eigenvector, with eigenvalue $\lambda_j = \mu_j^2$ and $T_{MS}$ has $f_j$ as an eigenvector, with eigenvalue $( \sqrt{\lambda_j} - 1)^2$. 
Finally, the fact that there are no other eigenfunctions also follows from \eqref{SVD}. 

The final part follows from the previous lemma. 
\end{proof}

\textbf{Equivalence of Barlow Twins loss to \cref{eq:ourloss}.}
The BarlowTwins loss from \citep{zbontar2021barlow} is as follows:
\begin{equation}
    \mathcal{L}_{BT} = \sum_i (C_{ii}-1)^2 + \beta \sum_i\sum_{j\neq i} C_{ij}^2
\end{equation}
where $C$ is the cross-correlation matrix computed between the outputs of the network to two different augmentations. First, the BarlowTwins loss can be seen as the unconstrained optimization form of the following constrained optimization objective:
\begin{equation}
    \mathcal{L}_{BT} = \sum_i (C_{ii}-1)^2 \quad \text{, subject to } \quad C_{ij} = 0 \quad \forall j \neq i
\end{equation}
where $\beta$ is the Lagrangian multiplier \citep{boyd2004convex}. In \citep{zbontar2021barlow}, the cross-correlation matrix $C$ is computed by a dot product between normalized functions $f_i$'s such that $(f_i, f_i)_\Ltr = 1$  $\forall i$. The network output for one augmentation of $x$, $a$, can be thought of as a Monte-Carlo estimate (with one sample) of $T_M f_i(x)$, where $f_i$ is the $i^{th}$ dimension of the network's output. 
Therefore, the BarlowTwins loss can be written in its following equivalent form:
\begin{equation}\label{eq:BTloss}
\hat{L}^{BT}(F) = \sum_{i=1}^{N_k} 
\left ( (T_M f_i,T_M f_i)_\Ltr -1\right)^2 \quad \text{,  subject to }\quad (f_i, f_j)_\Ltr = \delta_{ij} %\nonumber
\end{equation}
As shown by \citep{zhai2023understanding}, the eigenvalues of $T_M^TT_M$ are always less than 1. Therefore, we do not need the square in \Cref{eq:BTloss}. Rewriting it, we get the following:
\begin{equation}\label{eq:BTloss2}
\hat{L}^{BT}(F) = \sum_{i=1}^{N_k} 
 (T_M f_i,T_M f_i)_\Ltr \quad \text{,  subject to }\quad (f_i, f_j)_\Ltr = \delta_{ij} %\nonumber
\end{equation}
Using \cref{thm.c3}, we show that the loss recovers the equivalent eigenfunctions for the following reason. We can rewrite the loss as
\begin{align}
\hat{L}^{BT}(F) &= \sum_{i=1}^{N_k} 
((T_M-I) f_i,(T_M-I) f_i)_\Ltr  \quad \text{,  subject to }\quad (f_i, f_j)_\Ltr = \delta_{ij} \nonumber\\
\implies \hat{L}^{BT}(F) &= \sum_{i=1}^{N_k} \|T_M f_i - f_i\|^2_\Lt  \quad \text{,  subject to }\quad (f_i, f_j)_\Ltr = \delta_{ij} \nonumber\\
\end{align}
which recovers the loss \cref{eq:ourloss}.
Note that the VICReg loss \citep{bardes2021vicreg}, in addition to the constraints imposed by the BarlowTwins loss, ensures that the norm of $f_i$'s are more than some threshold. This can be easily incorporated into the constraint with a constant along with $\delta_{ij}$. In conclusion, both BarlowTwins and VICReg losses can be seen as equivalent forms of the loss \cref{eq:ourloss}.

\begin{theorem}
    (Informal) Let us denote the span of the feature space at initialization as $V(F_0)$ and after training as $V(F_T)$.
    For small initialization of the network's weights, the alignment of $V(F_T)$ with the eigenfunctions of $\Tau$ depend on two factors: (i) alignment of $V(F_0)$ with the eigenfunctions of $\Tau$; (ii) singular values of $\Tau$.
    \label{theorem:dynamics_appendix}

    \textbf{Theorem B.4. (Formal)} Let $\Gamma = V\Lambda V^T$ represent the eigendecomposition of $\Gamma$, and define $z$ as the projection of the weight vectors in $W$ onto singular vectors of $\Gamma$, V. Formally, $z = WV$. Assuming small initialization (as in Simon et al. (2023), i.e. $| z_{pi}(0) | << 1$ for all $p,i$, we can derive the following conclusions:

    \begin{enumerate}
        \item $sign(\frac{\Delta z_{pi}(t)}{z_{pi}(t)}) = sign(\lambda_i)$
        \item For all $\lambda_i, \lambda_j > 0$, $\frac{z_{pi}(t)}{z_{pi}(0)} = (\frac{z_{pj}(t)}{z_{pj}(0)})^{\frac{\lambda_i}{\lambda_j}}$ where $\lambda_i$ denotes the $i^{th}$ singular value, i.e. $i^{th}$ element of diagonal matrix $\Lambda$.
    \end{enumerate}
\end{theorem}
\begin{proof}
    We will first show that the above holds for a linear network, i.e. the output of the network with weights $W \in \mathbf{R}^{m\times n}$ is $WX$ for some input $X \in \mathbf{R}^{n \times b}$, where $m$ is the output dimensionality, $n$ is the input dimensionality and $b$ is the batch size. \\
    Let us first analytically compute the cross-correlation matrix $C$ following \citep{simon2023stepwise}.
    \begin{align}
        C = WXX'^TW^T &= W\Tau W^T \nonumber \\
        C_{pq} = \sum_{i,j} W_{pi}\Tau_{ij}W_{qj} \quad,&\quad C_{pp} = \sum_{i,j} W_{pi}\Tau_{ij}W_{pj} \nonumber
    \end{align}
    where $X$ and $X'$ are matrices $\in \mathbf{R}^{n \times b}$ containing two augmentations of a each image in a batch of images. Also, we have defined $\Tau = XX'^T$, i.e. the augmentation-defined data correlation matrix.\\
    Rewriting the BarlowTwins loss function from \citep{zbontar2021barlow}:
    \begin{equation}
        \mathcal{L}_{BT} = \sum_i (C_{ii}-1)^2 + \beta \sum_i\sum_{j\neq i} C_{ij}^2 \nonumber
    \end{equation}
    To study the learning dynamics, we need to compute the gradient of $\mathcal{L}_{BT}$ w.r.t. the parameters $W$. 
    \begin{equation}
        \label{eq:loss_bt_grad}
        \frac{d W_{pq}}{dt} = -\eta \frac{\partial \mathcal{L}_{BT}}{\partial W_{pq}} = -2\eta \sum_i (C_{ii}-1)\frac{\partial C_{ii}}{\partial W_{pq}} - 2 \eta \beta \sum_i \sum_{j\neq i} C_{ij}\frac{\partial C_{ij}}{\partial W_{pq}}
    \end{equation}
    Let us now analytically compute the derivatives of $C_{ii}$ and $C_{ij}$ w.r.t $W_{pq}$ to simplify each of the terms in \Cref{eq:loss_bt_grad}.
    \begin{align}
        \frac{\partial C_{ii}}{\partial W_{pq}} &= \frac{\partial}{\partial W_{pq}} \sum_{j,k} W_{ij}\Tau_{jk}W_{ik} = \frac{\partial}{\partial W_{pq}} \sum_{j,k} W_{ij}\Tau_{jk}W_{ik} \delta_{pi} \nonumber\\
        &= \left( \sum_{j,k} \Tau_{jk} W_{pk}\delta_{jq} + \sum_{j,k} W_{pj}\Tau_{jk} \delta_{kq} \right) \delta_{pi} \nonumber\\
        &= \left( \sum_k \Tau_{qk}W_{pk} + \sum_j W_{pj}\Tau_{jq} \right) \delta_{pi} \nonumber\\
        &= 2 \left[W\Tau \right]_{pq} \delta_{pi} \nonumber\\
        \implies \sum_i (C_{ii}-1)\frac{\partial C_{ii}}{\partial W_{pq}} &= 2(C_{pp}-1)\left[W\Tau \right]_{pq}
        \label{eq:loss_bt_grad_term1}
    \end{align}
    Using similar algebra steps, we can simplify the second term:
    \begin{align}
        \frac{\partial C_{ii}}{\partial W_{pq}} &= \left[W\Tau \right]_{jq} \delta_{pi} + \left[W\Tau \right]_{iq} \delta_{pj} \nonumber\\
        \implies \sum_i \sum_{j\neq i} C_{ij}\frac{\partial C_{ii}}{\partial W_{pq}} &= \sum_i \sum_{j\neq i} C_{ij} \left( \left[W\Tau \right]_{jq} \delta_{pi} + \left[W\Tau \right]_{iq} \delta_{pj} \right) \nonumber\\
        &= \sum_{j \neq q} C_{pj}\left[W\Tau \right]_{jq} + \sum_{i \neq q} C_{ip}\left[W\Tau \right]_{iq} \nonumber\\
        &= 2\left[(C-I)W\Tau \right]_{pq} - 2(C_{pp} -1)\left[W\Tau \right]_{pq}
        \label{eq:loss_bt_grad_term2}
    \end{align}
    Substituting \Cref{eq:loss_bt_grad_term1,eq:loss_bt_grad_term2} into \Cref{eq:loss_bt_grad}, we get:
    \begin{align}
        \frac{d W_{pq}}{dt} &= -\eta \frac{\partial \mathcal{L}_{BT}}{\partial W_{pq}} = -4 \eta (C_{pp} - 1)\left[W\Tau \right]_{pq} - 4\eta\beta\left[(C-I)W\Tau \right]_{pq} + 4\eta\beta (C_{pp} -1)\left[W\Tau \right]_{pq} \nonumber\\
        &= -4\eta (1-\beta)(C_{pp} - 1)\left[W\Tau \right]_{pq} - 4\eta\beta \left[(C-I)W\Tau \right]_{pq}
        \label{eq:loss_bt_grad_expr2}
    \end{align}
    Note that setting $\beta = 1$ yields the dynamics equation presented by \citep{simon2023stepwise}. However, in practice, $\beta$ is orders of magnitude less that 1. For sake of simplicity, we will analyze the extreme case of $\beta=0$, which will yield us insights into the weak-orthogonality constraint case. Therefore,
    \begin{equation}
        \frac{d W_{pq}}{dt} \approx -4 \eta (C_{pp} - 1)\left[W\Tau \right]_{pq}
        \label{eq:loss_bt_grad_approx}
    \end{equation}
    Let us denote the eigendecomposition of $\Tau$ be written as $\Tau = V\Lambda V^T$. Here, $\Lambda$ is a diagonal matrix with singular values as the diagonal elements. Let us also denote the projection of the weight vectors onto the singular vectors of $\Tau$, i.e. $V$ as $z$. So, $z = WV$.\\
    Therefore, using these definitions, we can write the following:
    \begin{align}
        C_{pp} &= \left[W\Tau W^T\right]_{pp} = \left[Z\Lambda Z^T\right]_{pp} = \sum_i z_{pi}^2\lambda_i \nonumber\\
        W\Tau &= WV\Lambda V^T = Z\Lambda V^T \nonumber
    \end{align}
    Now, writing the update equations \Cref{eq:loss_bt_grad_approx} in terms of $z_{pi}$:
    \begin{align}
        \frac{d z_{pi}}{dt} &= \sum_q \frac{d W_{pq}}{dt} V_{qi} \nonumber\\
        &= -4 \eta \left(\sum_j z_{pj}^2\lambda_j - 1 \right) \sum_k z_{pk}\lambda_k (\sum_q V_{qk} V_{qi}) \nonumber\\
        &= -4 \eta \left(\sum_j z_{pj}^2\lambda_j - 1\right)z_{pi}\lambda_i
        \label{eq:loss_bt_grad_approx_expr2}
    \end{align}
    Assuming small initialization of weights $W$, we can assume that $\mid z_{pi}(0)\mid << 1$, i.e. magnitude $z_{pi}$ at time 0 is very small.\\
    Let us define $h_p(t) = 1 - \sum_j z_{pj}(t)^2\lambda_j$. For small initialization, $h_p(t) > 0$  $\forall t$. Therefore,
    \begin{equation}
        sign\left(\frac{d z_{pi}(t)}{dt} \frac{1}{z_{pi}}\right) = sign(\lambda_i)
        \label{eq:loss_bt_grad_sign}
    \end{equation}
    It is clear from \Cref{eq:loss_bt_grad_sign} that if $\lambda_i < 0$, $\lim_{t \to \infty}z_{pi}(t) = 0$. Similarly, if $\lambda_i = 0$, then $z_{pi}(t) = z_{pi}(0)$ $\forall t$.\\
    Therefore, akin to the conclusions of \citep{simon2023stepwise}, the BarlowTwins loss recovers directions corresponding to positive singular values in the augmentation-defined covariance matrix, $\Tau$ and suppresses directions corresponding to negative singular values. Thus, the network outputs span the top singular vectors of $\Tau$.\\
    It is worth noting from \Cref{eq:loss_bt_grad_approx_expr2} that the following holds:
    \begin{align}
        \frac{1}{\lambda_i} \frac{d log(z_{pi})}{dt} &= \frac{1}{\lambda_j} \frac{d log(z_{pj})}{dt} \nonumber \\
        \implies \frac{1}{\lambda_i} log\left(\frac{z_{pi}(t)}{z_{pi}(0)}\right) &= \frac{1}{\lambda_j} log\left(\frac{z_{pj}(t)}{z_{pj}(0)}\right) \nonumber \\
        \implies \frac{z_{pi}(t)}{z_{pi}(0)} &= \left(\frac{z_{pj}(t)}{z_{pj}(0)} \right)^{\frac{\lambda_i}{\lambda_j}}
    \end{align}
\end{proof}
Without loss of generality, if $\lambda_i << \lambda_j$, then $z_{pi}(t) \approx z_{pi}(0)$. Therefore, under small initialization, i.e. $z_{pi}(0)$ is small $\forall i$, gradient descent biases the $p^{th}$ weight vector to be more strongly aligned to the eigenvector corresponding to the strongest eigenvalue, for all $p$'s. Hence, under weak orthogonalization constraints, the BarlowTwins loss will over "represent" the strong singular vectors of the augmentation-defined cross-correlation matrix.\\
When using high-dimensional projectors, specifically when $m >> \sum_i \mathbf{1}_{\lambda_i > 0}$, wherein $\mathbf{1}_{\zeta}$ is the indicator function that is 1 when condition $\zeta$ is true and 0 otherwise, this problem might be ameliorated because there are multiple weight vectors that might be aligned with the top singular vectors of $\Tau$ at initialization. However, when using low-dimensional projectors, we do not have such a luxury and therefore, using a weak orthogonalization constraint leads to dimensionality collapse in the representation space.

\textbf{Extending to deep non-linear networks.} Similar to the analysis in \citep{simon2023stepwise}, we can repeat the above analysis by replace $X$ and $X'$ by the corresponding kernel versions, where the kernel corresponds to the Neural Tangent Kernel (NTK) of the network. Therefore, the implicit bias of gradient descent to yield dimensionality collapse in the representation space when using weak orthogonalization constraints still remains.

\textbf{Dimensionality collapse under noisy optimization.} From the rest of this section, we have seen that the BarlowTwins loss is a Monte-Carlo estimate of the true data-augmentation defined covariance matrix. Moreover, stochastic gradient descent adds noise due to mini-batch sampling to the optimization process. Note that there exist symmetries in our linear network, i.e. an orthogonal rotation of the weight matrix yields the same loss function. As explained in \citep{chen2023stochastic}, such symmetry-invariant sets are potential candidates for stochastic collapse when performing noisy gradient-based optimization. Therefore, the presence of noise in the data-augmentation covariance matrix, $\Tau$, as well as the batch noise can further worsen the dimensionality collapse problem where different weight vectors become parallel to each other due to noise in updates. One possible mitigation strategy is to obtain a better estimate of the true augmentation-defined covariance matrix (see \Cref{fig:Ansatz}), which we discuss in the next section.
% \end{comment}

\textbf{Empirical validation.} We empirically validate our results on the learning dynamics on simplistic 2-dimensional settings. These results, demonstrating the difference in feature learning dynamics for weak vs strong orthogonalization, are presented as GIFs in the supplementary material, and can also be viewed at the \href{https://sites.google.com/view/harnessing-small-projectors}{project website}.

% \bibliographystyle{alpha}
% \bibliography{Bib}

% \end{document}

\section{Multi-Augmentation Learning}
\label{sec:multi_augs_theory}
% In this section, we present a theoretical explanation of why multiple patch augmentations improve learning the invariance properties in the feature space. We use the framework of an augmentation graph \citep{haochen2021provable,shen2022connect} and study how insufficiently sampling this graph could lead to pathological learning. 

\subsection{Augmentation graph}
We use the population augmentation graph formulation introduced in \cite{haochen2021provable}. Briefly, we define a graph $\mathcal{G(X,W)}$, where the vertex set $\mathcal{X}$ comprises of all augmentations from the dataset (could be infinite when continuous augmentation functions are used) and $\mathcal{W}$ denotes the adjacency matrix with edge weights as defined below:
\begin{equation}
    w_{xx'} := \mathbb{E}_{\bar{x} \sim \mathcal{P}_{\bar{X}}} \left[ \mathcal{A}(x|\bar{x})\mathcal{A}(x'|\bar{x}) \right]
\end{equation}
, i.e. the joint probability of generating `patches' $x,x'$ from the same image $\bar{x}$. Here $\mathcal{A}$ defines the set of augmentation functions used in the SSL pipeline. It is worth noting that the magnitude of $w_{xx'}$ captures the relative similarity between $x$ and $x'$. A higher value of $w_{xx'}$ indicates that it is more likely that both patches came from the same image, and thereby are more similar. 
The marginal likelihood of each patch $x$ can also be derived from this formulation:
\begin{equation}
    w_x = \mathbb{E}_{x' \sim \mathcal{X}} \left[w_{xx'} \right]
\end{equation}

\begin{figure}[t]
    \centering
    \includegraphics[width=0.9\textwidth]{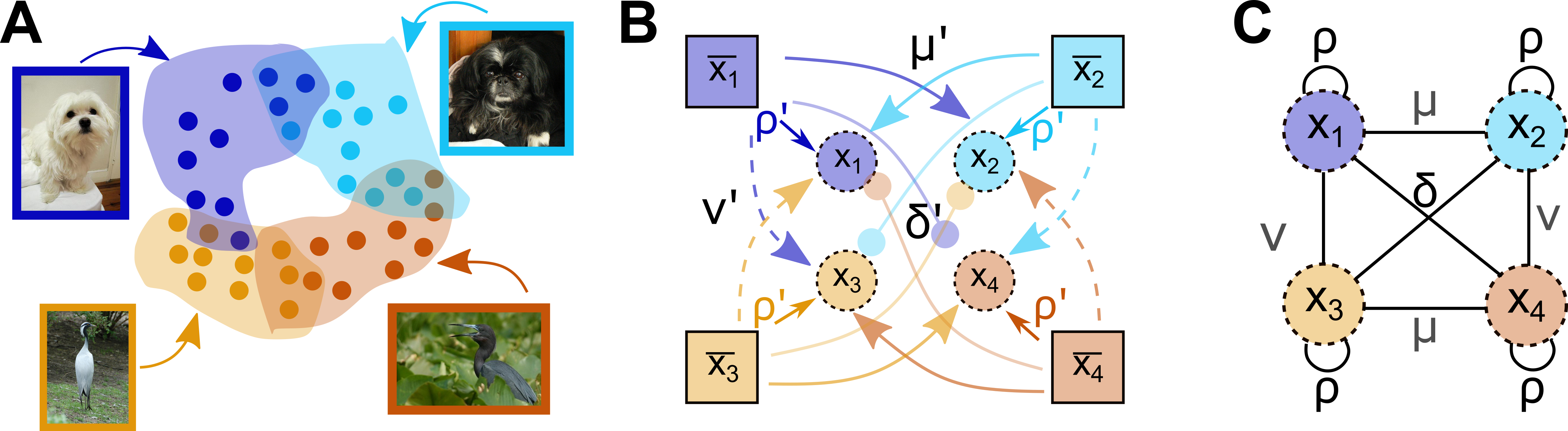}
    \caption{Schematic of augmentation graph. (A) Augmentations from each image span a region in the image space which could overlap with the augmentation span of other images. (B) An augmentation graph schematic that uses probabilities to characterize the interactions among augmentation spans of different instances.}
    \label{fig:aug_graph}
\end{figure}

\subsection{Contrastive and non-contrastive losses suffer from the same issues}
We will now show that the proposal of using multiple patches for the $\mathcal{L}_{invariance}$ is pertinent to both the contrastive and non-contrastive SSL.
Following \citep{haochen2021provable}, we use the spectral contrastive loss formulation and incorporate the augmentation graph relations:
\begin{align}
    \mathcal{L}_c &= - \mathbb{E}_{x,x^+} \left[f(x)^T f(x^+) \right] + \beta \mathbb{E}_{x,x'} \left[ \left(f(x)^T f(x') \right)^2 \right] \nonumber \\
    \mathcal{L}_c &\propto \|ZZ^T - D^{-\frac{1}{2}} \mathcal{W} D^{-\frac{1}{2}} \|^2_F = \|ZZ^T - \bar{\mathcal{W}} \|^2_F 
    \label{eq:L_con}
\end{align}
where $z := \sqrt{w_x}f(x)$, $D$ is a $N \times N$ diagonal matrix with entries $\{w_x\}$ and $\bar{\mathcal{W}} = D^{-\frac{1}{2}} \mathcal{W} D^{-\frac{1}{2}}$.

We extend the duality results between contrastive and non-contrastive SSL loss, established by \citep{garrido2022duality}, to demonstrate how \cref{eq:L_con} can be decomposed into the invariance and collapse-preventing loss terms. 
\begin{align}
    % \|ZZ^T - \bar{\mathcal{W}} \|^2_F &= Tr\left[Z^TZZ^TZ \right] + 2Tr\left[ Z^T(I_N-\bar{\mathcal{W}})Z\right] - 2Tr\left[Z^TZ\right] + Tr\left[I_d\right] + \kappa \nonumber \\
    % &= Tr\left[(Z^TZ - I_d)^T (Z^TZ - I_d)\right] + 2Tr\left[ Z^T(I_N-\bar{\mathcal{W}})Z\right] + \kappa \nonumber \\
    \|ZZ^T - \bar{\mathcal{W}} \|^2_F &= \|Z^TZ - I_d\|^2_F + 2Tr\left[ Z^T(I_N-\bar{\mathcal{W}})Z\right] + \kappa \label{eq:L_noncon1} \\
    &= \|Z^TZ - I_d\|^2_F + 2\sum_i \sum_x (1-\bar{w}_x) z_i^2 - 2\sum_i \sum_{x,x'} \bar{w}_{xx'} z_i z_i' + \kappa
    \label{eq:L_noncon2}
\end{align}
where $\kappa$ is some constant independent of $Z$. The first term in \cref{eq:L_noncon1} is the covariance regularization term in non-contrastive losses like BarlowTwins (implicit) or VIC-Reg (explicit), and the second term in \cref{eq:L_noncon2} is the variance regularization.
Simplifying the third term in \cref{eq:L_noncon2} gives us:
\begin{align}
    \sum_i \sum_{x,x'} \bar{w}_{xx'} z_i z_i' &= \sum_i \sum_{x,x'} w_{xx'} f(x)_i f(x')_i = \sum_i \sum_{x,x'} \mathbb{E}_{\bar{x} \sim \mathcal{P}_{\bar{X}}} \left[ \mathcal{A}(x|\bar{x})\mathcal{A}(x'|\bar{x}) f(x)_i f(x')_i\right] \nonumber \\
    % &= \sum_i \mathbb{E}_{\bar{x} \sim \mathcal{P}_{\bar{X}}} \left[ \sum_{x,x'} \mathcal{A}(x|\bar{x})\mathcal{A}(x'|\bar{x}) f(x)_i f(x')_i \right] \nonumber \\
    &= \sum_i \mathbb{E}_{\bar{x} \sim \mathcal{P}_{\bar{X}}} \left[\sum_{x}\mathcal{A}(x|\bar{x}) (f(x)_i \overline{f(x)}_i  - f(x)_i^2)\right] \nonumber \\
    &= \mathbb{E}_{\bar{x} \sim \mathcal{P}_{\bar{X}}} \left[\sum_{x}\mathcal{A}(x|\bar{x}) \left(f(x)^T\overline{f(x)} - \|f(x)\|^2\right)\right]
    \label{eq:L_noncon3}
\end{align}
This term encourages $f(x)$ to be similar to $\overline{f(x)}$, i.e. the mean representation across all augmentations of $\bar{x}$, thereby requiring to ``sufficiently'' sample $A(.|\bar{x})$. Given that both the contrastive and non-contrastive losses rely on learning invariance properties from data augmentations, we believe that our multi-patch proposal would improve the probability density estimation of $A(.|\bar{x})$ and yield better performance with few training epochs.

\subsection{Explaining training dynamics in low patch sampling regime}
We now turn to a simple form of the augmentation graph to understand how using low number of augmentations affects the evolution of $ZZ^T$. Minimizing \cref{eq:L_con} implies that the spectral decomposition of $Z$ would align with the top eigenvectors (and values) of $\overline{\mathcal{W}}$. We will demonstrate that in the low sampling regime (using few augmentations), the eigenvectors of the sampled augmentation graph $\Tilde{\mathcal{W}}$ \textit{may not} align with those of $\overline{\mathcal{W}}$.

\textbf{Augmentation graph setup.} We define an augmentation graph with only two instances from two different classes, similar to the one presented in \citep{shen2022connect}. Let us denote the four instances as $\bar{x}_i$ for $i \in {1,2,3,4}$, where $\bar{x}_1,\bar{x}_2$ belong to class 1 (i.e. $y_1,y_2 = 1$) and $\bar{x}_3,\bar{x}_4$ belong to class 2 (i.e. $y_3,y_4 = 4$). Let us further assume that $\bar{x}_1,\bar{x}_3$ have the highest pixel-level similarity among $(\bar{x}_1,\bar{x}_i) \forall i \in {2,3,4}$, thereby making it more likely to have similar patches. We denote this relationship among input examples using $\mathcal{G}$ to indicate (pixel-wise) global similarity groups. So, $\mathcal{G}_1,\mathcal{G}_3 = 1$ and $\mathcal{G}_2,\mathcal{G}_4=2$. We can use the following probabilistic formulation to model our augmentation functions (see \Cref{fig:aug_graph}B):

\begin{align}
    A(x_j | \bar{x}_i) =
    \begin{cases}
        \rho' & \text{if $j=i$} \\
        \mu' & \text{if $j\neq i$ and $y_j = y_i$ and $\mathcal{G}_j \neq \mathcal{G}_i$} \\
        \nu' & \text{if $j\neq i$ and $y_j \neq y_i$ and $\mathcal{G}_j = \mathcal{G}_i$} \\
        \delta' & \text{if $j\neq i$ and $y_j \neq y_i$ and $\mathcal{G}_j \neq \mathcal{G}_i$}
    \end{cases}
\end{align}
In our setting, $\rho'+\mu'+\nu'+\delta' = 1$. 
% Given this augmentation function, we can write the expression for the augmentation graph. For instance, assuming all input examples are equally likely,
% \begin{align}
%     w_{11} &= \mathbb{E}_{\overline{x}_i} [A(x_1|\overline{x}_i)A(x_1|\overline{x}_i)] \nonumber \\
%     &= \frac{1}{4}(\rho'^2 + \mu'^2 + \nu'^2 + \delta'^2) = \rho \quad \text{(say)} \nonumber
% \end{align}
The adjacency matrix of our augmentation graph (as shown in \Cref{fig:aug_graph}C) is as follows:
\begin{align}
    \overline{\mathcal{W}} =
    \begin{bmatrix}
        \rho & \mu & \nu & \delta \\
        \mu & \rho & \delta & \nu \\
        \nu & \delta & \rho & \mu \\
        \delta & \nu & \mu & \rho \\
    \end{bmatrix}
\end{align}
We defer the relations between $\rho',\mu',\nu'\delta'$ and $\rho,\mu,\nu,\delta$ to the appendix. The eigenvalues of this matrix are: ($\rho + \mu + \nu + \delta$, $\rho + \mu - \nu - \delta$,$\rho - \mu + \nu - \delta$, $\rho -\mu - \nu + \delta$). Corresponding eigenvectors are along $\left[1,1,1,1 \right]^T$, $\left[1,1,-1,-1 \right]^T$. $\left[1,-1,1,-1 \right]^T$, $\left[1,-1,-1,1 \right]^T$. 
Assuming that the augmentation functions induce semantically-relevant invariance properties that are relevant for identifying $y_i$ from $f(x_i)$, we can say that $\rho' > max\{\mu',\nu'\}$ and $min\{\nu',\mu'\} > \delta'$. When we have sufficiently sampled the augmentations, any SSL loss will learn $Z$ such that its singular values are span the top eigenvectors of the augmentation graph, and the eigenspectrum of $ZZ^T$ would simply be the above eigenvalues. 
In practical settings, the augmentation graph would have significantly higher dimension that the feature/embedding dimension \footnote{Contrastive algorithms use a large batch size, thereby optimizing a high-dimensional $ZZ^T$ whereas non-contrastive algorithms use a large embedding dimension, thereby optimizing a high-dimensional $Z^TZ$.}. Therefore, singular vectors of $Z$ would span the top eigenvectors of $\overline{\mathcal{W}}$ and the smaller eigenmodes are not learned. 
% It is clear that $\lambda_1 = (\rho+\mu+\nu+\delta)$ is the largest eigenvalue, followed by $\lambda_2 = (\rho-\delta+|\mu-\nu|)$, $\lambda_3 = (\rho-\delta-|\mu-\nu|)$ and the smallest eigenvalue would be $\lambda_4=(\rho+\delta-\mu-\nu)$. To mimic a similar property, we assume that $Z \in \mathbb{R}^{4 \times d}$, where $d < 4$.
When we have accurately sampled the augmentation graph, $\mu > \nu$ and therefore, the class-information preserving information is preferred over pixel-level preserving information during learning. But what happens when we \textit{do not sufficiently sample the augmentation space?}

\begin{figure}[t]
    \centering
    \includegraphics[width=\textwidth]{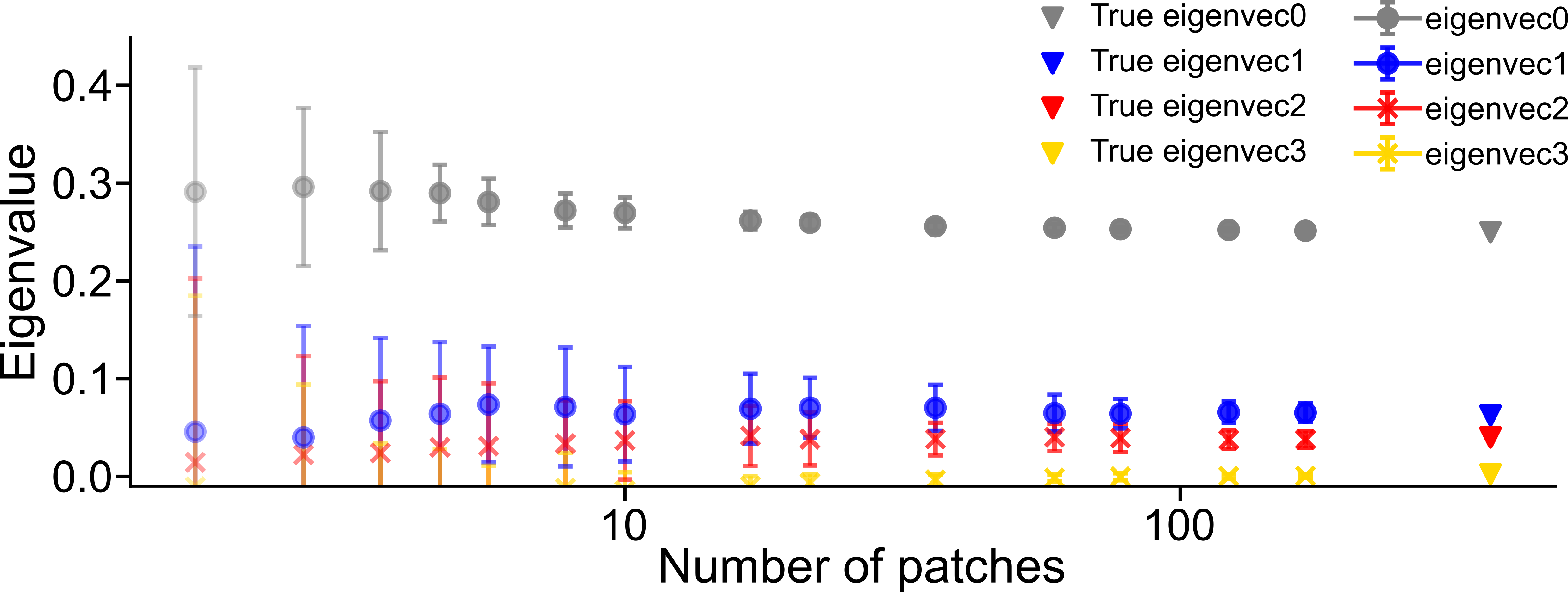}
    \caption{Empirical verification of the subsampling Ansatz.}
    \label{fig:Ansatz}
\end{figure}

\textbf{Ansatz.} Based on our empirical experience, we define \textit{an ansatz} pertaining to the eigenvalues of a sampled augmentation graph and validate it in tractable toy settings, such as the one described above. Specifically, we claim that when the augmentation space is not sufficiently sampled, $\{|\mu-\nu|,\delta\} \to 0$. In other words, we claim that when only few augmentations per example are used, it is more likely to have an equal empirical likelihood for augmentations that preserve (pixel-level) global information and class/context information. Moreover, it is very unlikely to have augmentations that change both the class and global information. This is demonstrated in \Cref{fig:Ansatz}.

\textbf{Consequences of the \textit{Ansatz}.} When only a few augmentations are sampled, learning can suppress the class information at the cost of preserving the pixel-level information, thereby leading to an increased smoothness in the learned feature space.

\section{Implementation Details}
\label{sec:implementation}
\textbf{Image Classification Datasets} Across all experiments, our settings mainly follow \cite{chen2022intra}. In particular, \Cref{tab:pretrain} summarizes our pretraining settings on Cifar-10 \cite{cifar10}, STL-10 \cite{coates2011analysis} and Imagenet-100 \citep{russakovsky2015imagenet}. The Imagenet-100 dataset was generated by sampling 100 classes from the original Imagenet-1k dataset, according to this \href{https://github.com/danielchyeh/ImageNet-100-Pytorch/blob/main/IN100.txt}{list} \cite{tian2020contrastive}. In \Cref{tab:finetune}, we outline the corresponding linear evaluation settings for Resnet-50 (for CIFAR-10 and STL-10) and ResNet-18 (for Imagenet). Note that we add a linear classifier layer to the encoder’s features and discard the projection layers for evaluation. Our code base is publicly available on \href{https://github.com/kumarkrishna/fastssl}{github}.

\begin{table}[ht]
  \centering
  \begin{minipage}[b]{0.5\linewidth} % Adjust the width as needed
    \centering
    \begin{tabular}{l|c} % Table 1
      config & value \\
      \hline
      optimizer & Adam  \\
      learning rate  & 1e-3   \\
      batch size   & 128 (Imagnet), 256 (CIFAR, STL)   \\
      epochs  &  100 \\
      weight-decay & 1e-6   \\
    \end{tabular}
    \subcaption{Pretraining}
    \label{tab:pretrain}
  \end{minipage}
  \hfill % Add some horizontal space between the tables
  \begin{minipage}[b]{0.45\linewidth} % Adjust the width as needed
    \centering
    \begin{tabular}{l|c} % Table 2
      config & value \\
      \hline
      optimizer & Adam \\
      learning rate  & 1e-3   \\
      batch size   & 512   \\
      epochs  &  200  \\
      weight-decay & 1e-6   \\
      test-patches & 16 \\
    \end{tabular}
    \subcaption{Linear Evaluation}
    \label{tab:finetune}
  \end{minipage}
  \caption{Experiment Protocol for comparing SSL algorithms}
\end{table}

The key SSL loss functions that we use in this work are BarlowTwins \citep{zbontar2021barlow} and VICReg \citep{bardes2021vicreg}. Let us suppose that the embeddings of two augmentations of a batch of images are denoted as $z$ and $z'$. The BarlowTwins loss function is as follows:
\begin{align}
    \label{eq:BT_implementation}
    \mathcal{L}_{BT} &= \sum_i (C_{ii}-1)^2 + \beta \sum_i\sum_{j\neq i} C_{ij}^2 \\
    \text{where }\quad C &= \frac{1}{n-1} \sum_{k=1}^{n} (z_k - \bar{z})(z_k' - \bar{z'})^T \nonumber \\
    \text{and }\quad \bar{z} &= \frac{1}{n} \sum_{k=1}^{n} z_k \quad \text{,} \quad \bar{z'} = \frac{1}{n} \sum_{k=1}^{n} z_k'
\end{align}
$C_{ij}$ is the element of $C$ at row $i$, column $j$ and $n$ is the batch size. For each projector dimensionality, $d$, we search for the hyperparameter, $\beta$, that yields the best downstream task performance.

The VICReg loss function is as follows:
\begin{align}
    \label{eq:VICReg_implementation}
    \mathcal{L}_{VIC} &= \frac{1}{n} \mu \sum_{k=1}^n \|z_k - z_k' \|^2 +   \frac{1}{2} \mu \left[v(Z) + v(Z') \right] + \frac{1}{2} \left[c(Z) + c(Z') \right]\\
    \text{where }\quad v(Z) &= \frac{1}{d} \sum_{i=1}^{d} max(0, 1 - Stdev(z_{:,i})) \nonumber \\
    \text{and }\quad c(Z) &= \frac{1}{d} \sum_i \sum_{j\neq i}\left[C(Z)_{ij}\right]^2 \quad \text{,} \quad C(Z) = \frac{1}{n-1} \sum_{k=1}^{n} (z_k - \bar{z})(z_k - \bar{z})^T \nonumber 
\end{align}
For each projector dimensionality, $d$, we search for the hyperparameter, $\mu$, that yields the best downstream task performance.

\subsection{Empirical results for low-dimensional projectors}

\begin{figure}[!ht]
    \centering
    \includegraphics[width=\linewidth]{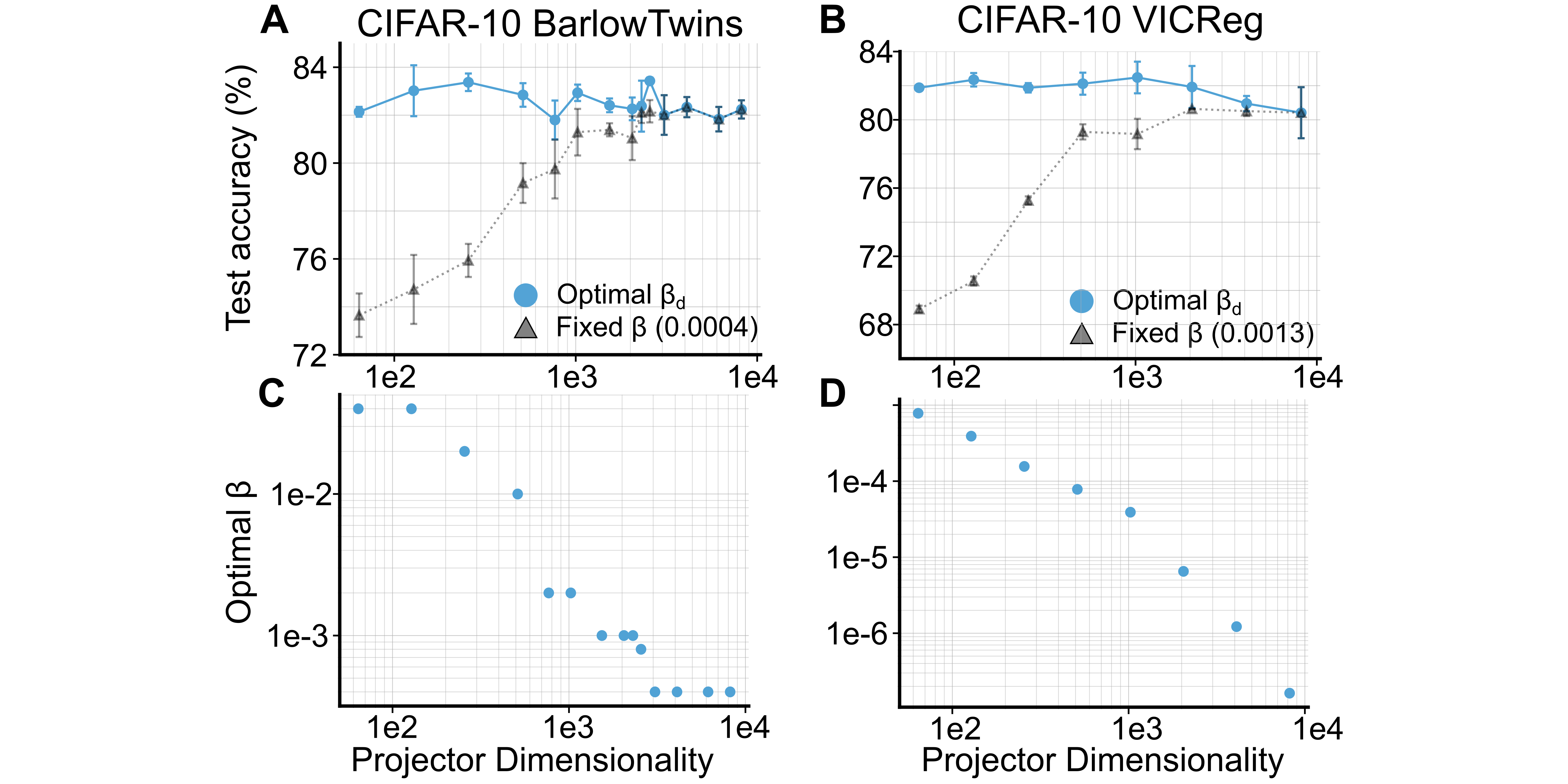}
    \vspace{-0.5cm}
    \caption{Low-dimensional projectors can yield good representations for both BarlowTwins and VICReg. We demonstrate that using a higher orthogonality constraint, $\beta$, for lower projector dimensionality can achieve similar performance over a wide range of projector dimensions ($d$). Note that for VICReg, we plot the ratio of the coefficient of the covariance loss to the coefficient of the invariance loss, i.e. $\beta = \frac{1}{d * \mu}$, where $\mu$ is the coefficient of the invariance loss. (See \Cref{eq:VICReg_implementation} for details of the loss formulation.)}
    \label{fig:fig3E_lambda_pdim}
\end{figure}

\begin{table}[!ht]
 \small
  \centering
  \begin{tabular}{c|c|c|c|c|c}
    \hline
    \multicolumn{1}{c|}{\multirow{2}{*}{{\it pdim}}} & \multicolumn{1}{c|}{\multirow{2}{*}{{\it Projector params (approx)}}} & \multicolumn{2}{c|}{\bf Barlow Twins} & \multicolumn{2}{c}{\bf VICReg} \\
    % \cline{2-3}
    \multicolumn{1}{c|}{} & \multicolumn{1}{c|}{} & fixed $\beta$ & optimal $\beta^*$  & fixed $\beta$ & optimal $\beta^*$ \\
    \hline
    % \vspace{-2mm} \\
    64 & 135k & 73.6 $\pm$ 0.9 & 82.1 $\pm$ 0.2 & 68.9 $\pm$ 0.2 & 81.9 $\pm$ 0.1 \\
    128 & 278k & 74.7 $\pm$ 1.4 & 83.0 $\pm$ 1.1 & 70.6 $\pm$ 0.3 & 82.3 $\pm$ 0.4 \\
    256 & 589k & 75.9 $\pm$ 0.7 & 83.4 $\pm$ 0.4 & 75.3 $\pm$ 0.2 & 81.9 $\pm$ 0.3 \\
    512 & 1.3M & 79.2 $\pm$ 0.8 & 82.8 $\pm$ 0.5 & 79.3 $\pm$ 0.4 & 82.1 $\pm$ 0.6 \\
    1024 & 3.1M & 81.3 $\pm$ 1.0 & 82.9 $\pm$ 0.3 & 79.2 $\pm$ 0.9 & 82.5 $\pm$ 0.9 \\
    2048 & 8.3M & 81.0 $\pm$ 0.9 & 82.3 $\pm$ 0.5 & 80.6 $\pm$ 0.0 & 81.9 $\pm$ 1.2 \\
    4096 & 25.2M & 82.3 $\pm$ 0.4 & 82.3 $\pm$ 0.4 & 80.5 $\pm$ 0.3 & 81.0 $\pm$ 0.4 \\
    8192 & 83.9M & 82.2 $\pm$ 0.4 & 82.2 $\pm$ 0.4 & 80.4 $\pm$ 1.5 & 80.4 $\pm$ 1.5 \\
    \hline
  \end{tabular}
  \vspace{2mm}
  \caption{Extended version of \Cref{tab:multicol_example}. Optimizing for orthogonality appropriately allows low-dimensional projectors to match the performance for BarlowTwins and VICReg (on CIFAR-10) of much higher-dimensional projectors.}
  \label{tab:cifar10_pdim_res_extended}
\end{table}
\clearpage
\subsection{Empirical results with multi-augmentations along with Time}

\begin{figure}[htbp]
    \centering
    \includegraphics[width=\linewidth]{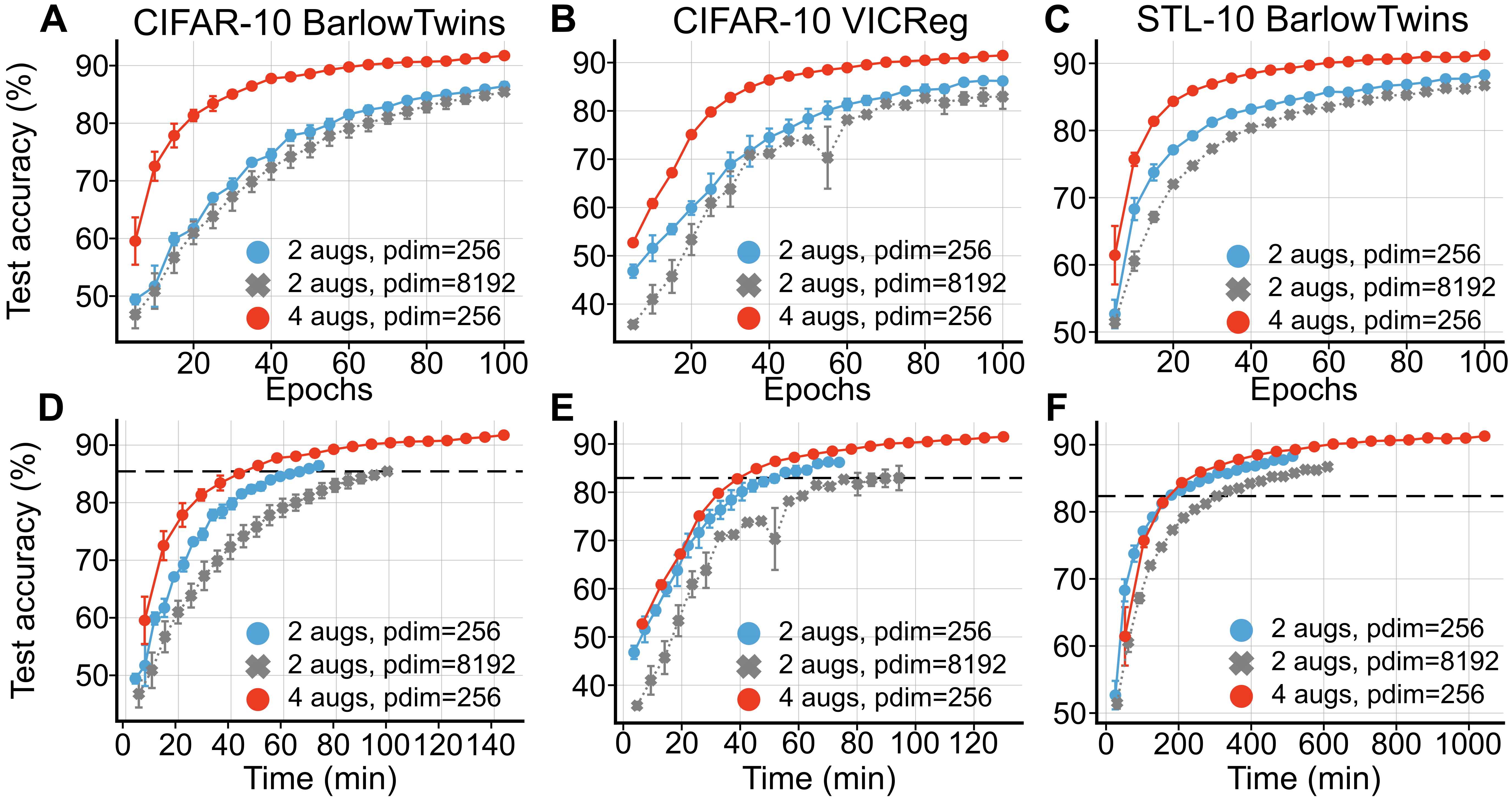}
    \vspace{-0.25cm}
    \caption{Using multiple augmentations improves representation learning performance and convergence. (A-C) Across BarlowTwins and VICReg for CIFAR-10 and STL-10 pretraining, using 4 augmentations instead of 2 helps improve performance. (D-F) Although the 4-augmentations take longer for each epoch, its performance still trumps the 2-augmentation version of the algorithm at the same wall clock time. Please see \Cref{sec:full_dset_pretraining} for more results.}
    \label{fig:multipatch_old}
\end{figure}

\begin{figure}[!ht]
    \centering
    \includegraphics[width=\linewidth]{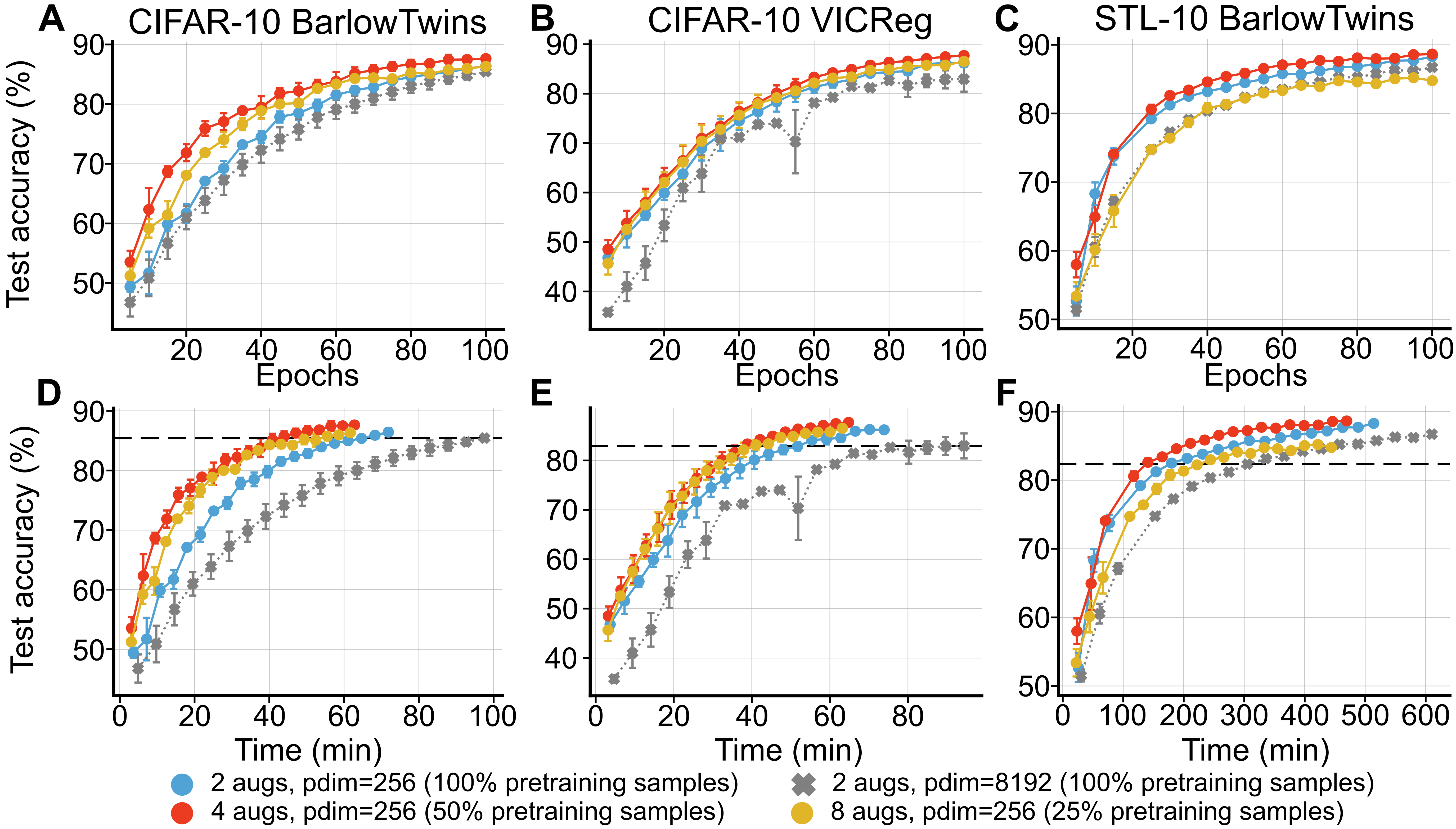}
    \caption{Multi-augmentation improves sample efficiency, recovering similar performance with significantly fewer unique samples in the pretraining dataset. Across BarlowTwins and VICReg pretraining on CIFAR-10 and STL-10, for the same effective dataset size ($\#augs \times \#unique\_samples$), using more patches improves performance at the same epoch (A-C) or wall clock time (D-F). However, a tradeoff exists wherein more data augmentations fail to improve performance in the scarce data regime.}
    \label{fig:multiPatch_fracTrain_old}
\end{figure}

% \subsection{Training on Imagenet-100 with 4 and 8 augmentations}
\begin{figure}[!htbp]
    \centering
    \includegraphics[width=0.8\linewidth]{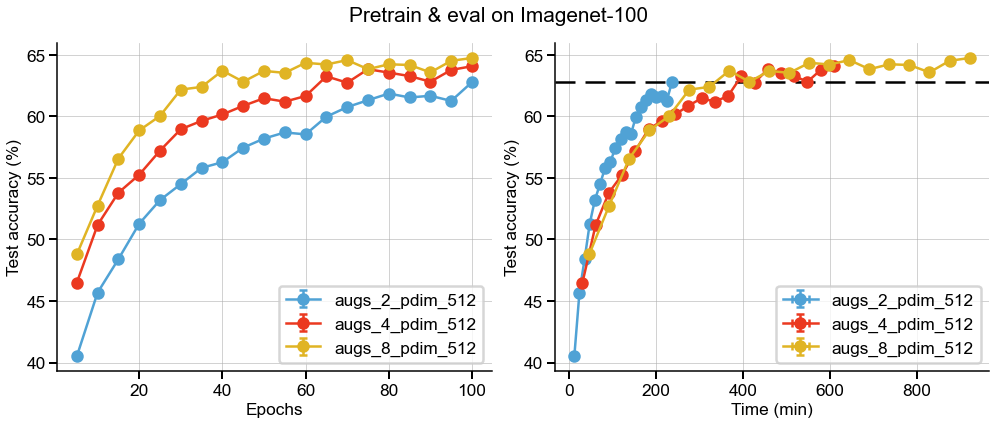}
    \caption{BarlowTwins pretraining on full Imagenet-100 dataset with 2, 4 and 8 augmentations.}
    \label{fig:rebutt_tALx2}
\end{figure}

\begin{figure}[!htbp]
    \centering
    \includegraphics[width=0.8\linewidth]{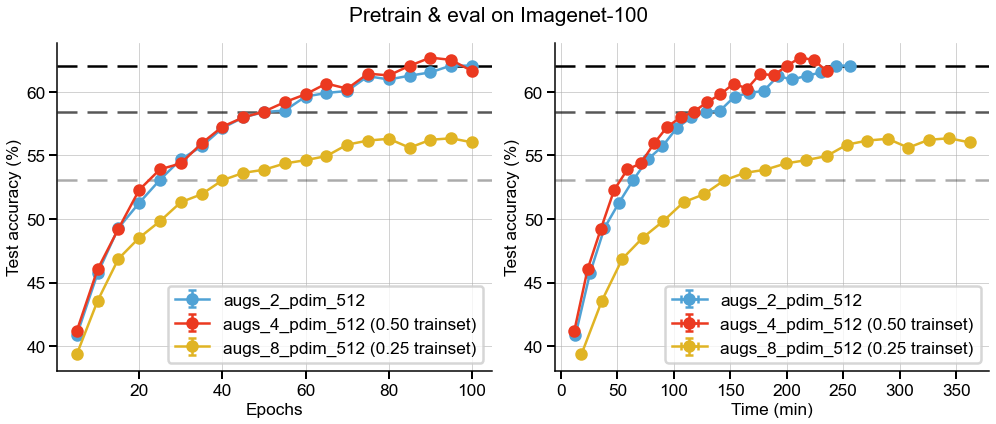}
    \caption{BarlowTwins pretraining on fraction of Imagenet-100 dataset with 2, 4 and 8 augmentations.}
    \label{fig:rebutt_tALx3}
\end{figure}
\clearpage

\subsection{Empirical results on transfer learning}
In this section, we present extended version of results presented in \Cref{fig:multipatch}, \Cref{fig:multiPatch_fracTrain} but pretraining on CIFAR-10 (or STL-10) and evaluating on STL-10 (or CIFAR-10). These results, coupled with the ones in \Cref{fig:multipatch} \Cref{fig:multiPatch_fracTrain}, present a strong case for the advantage of using the proposed multi-augmentation loss for better convergence as well as downstream accuracy.

\begin{figure}[!htbp]
    \centering
    \includegraphics[width=0.8\linewidth]{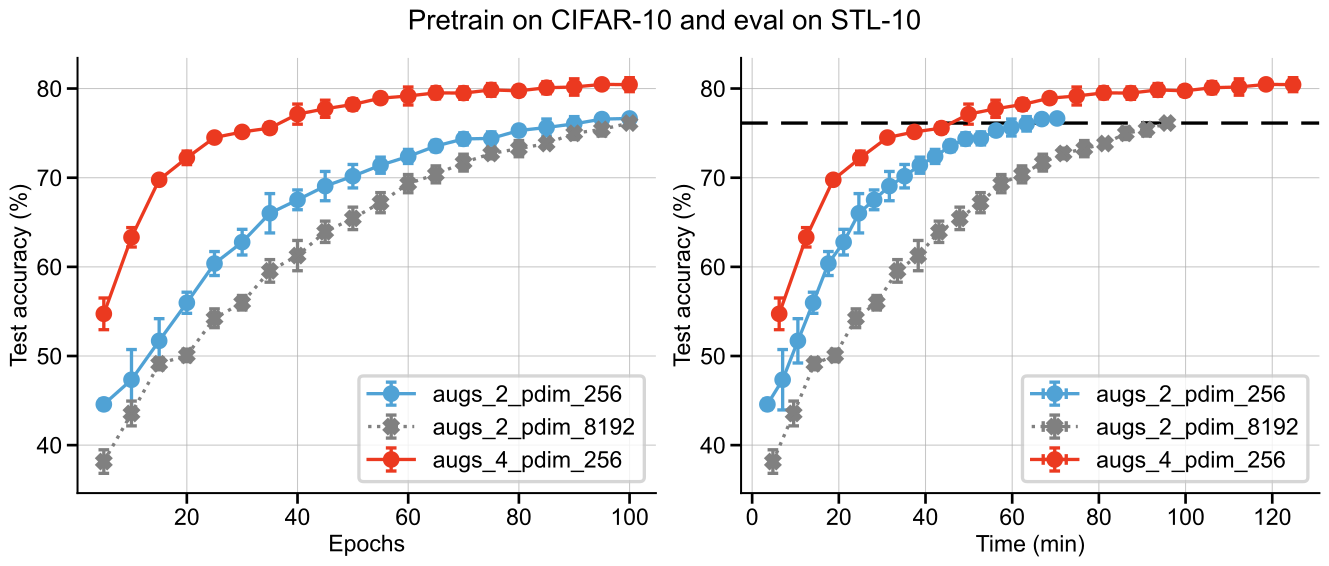}
    \caption{BarlowTwins pretraining on CIFAR-10, linear evaluation on STL-10 labelled set.}
    \label{fig:fig4A_CIFAR_Barlow}
\end{figure}

\begin{figure}[!htbp]
    \centering
    \includegraphics[width=0.8\linewidth]{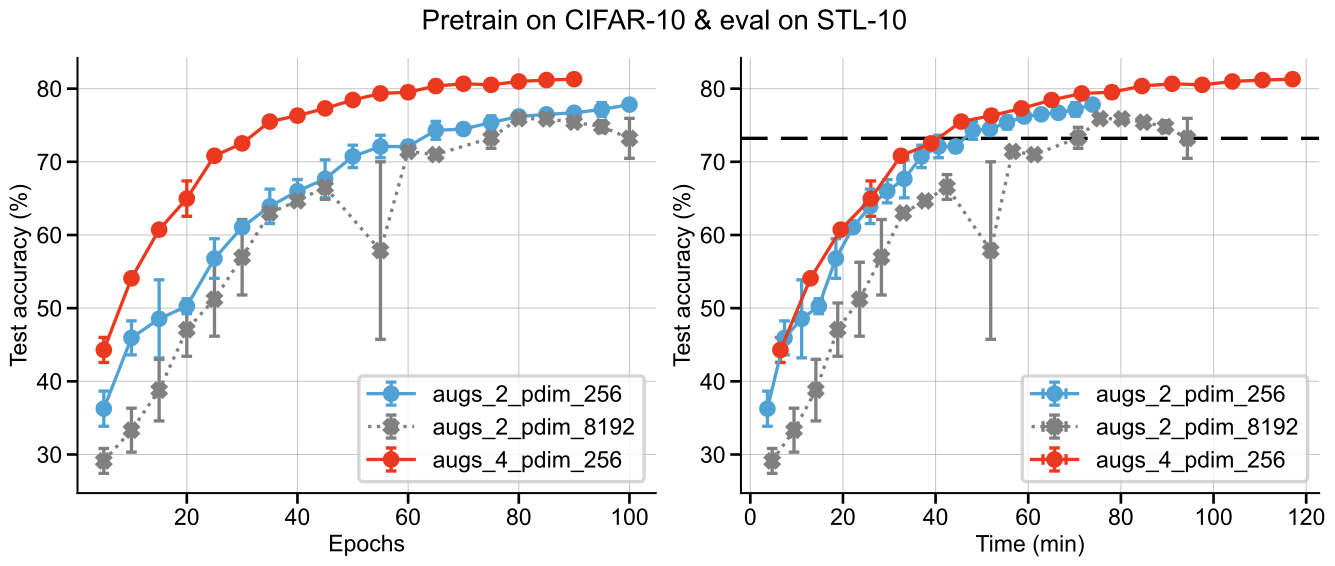}
    \caption{VICReg pretraining on CIFAR-10, linear evaluation on STL-10 labelled set.}
    \label{fig:fig4A_CIFAR_VICReg}
\end{figure}

\begin{figure}[!htbp]
    \centering
    \includegraphics[width=0.8\linewidth]{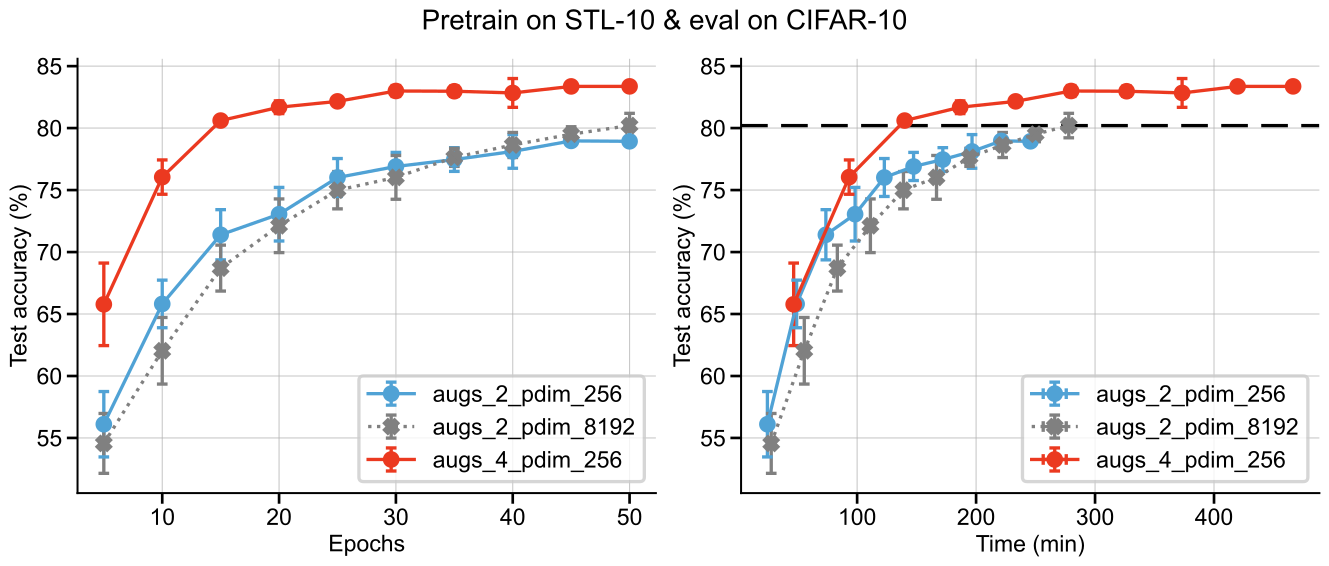}
    \caption{BarlowTwins pretraining on STL-10, linear evaluation on CIFAR-10 labelled set.}
    \label{fig:fig4A_STL_Barlow}
\end{figure}

\begin{figure}[!htbp]
    \centering
    \includegraphics[width=0.8\linewidth]{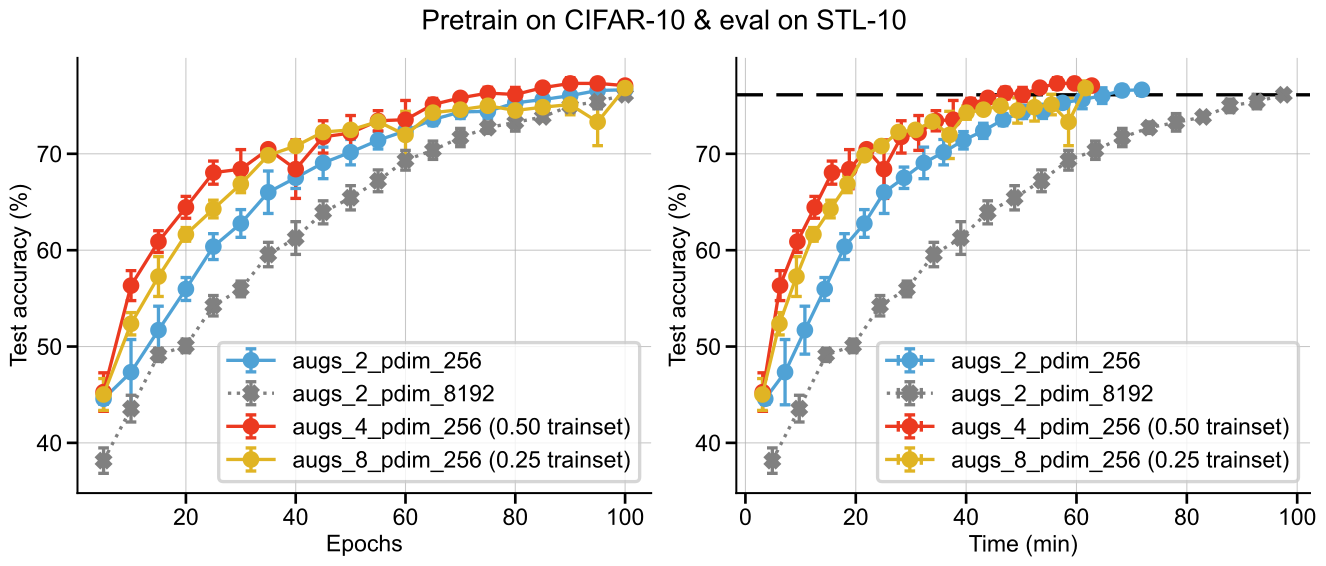}
    \caption{BarlowTwins pretraining on fraction of CIFAR-10 trainset, linear evaluation on STL-10 labelled set.}
    \label{fig:fig5A_CIFAR_Barlow}
\end{figure}

\begin{figure}[!htbp]
    \centering
    \includegraphics[width=0.8\linewidth]{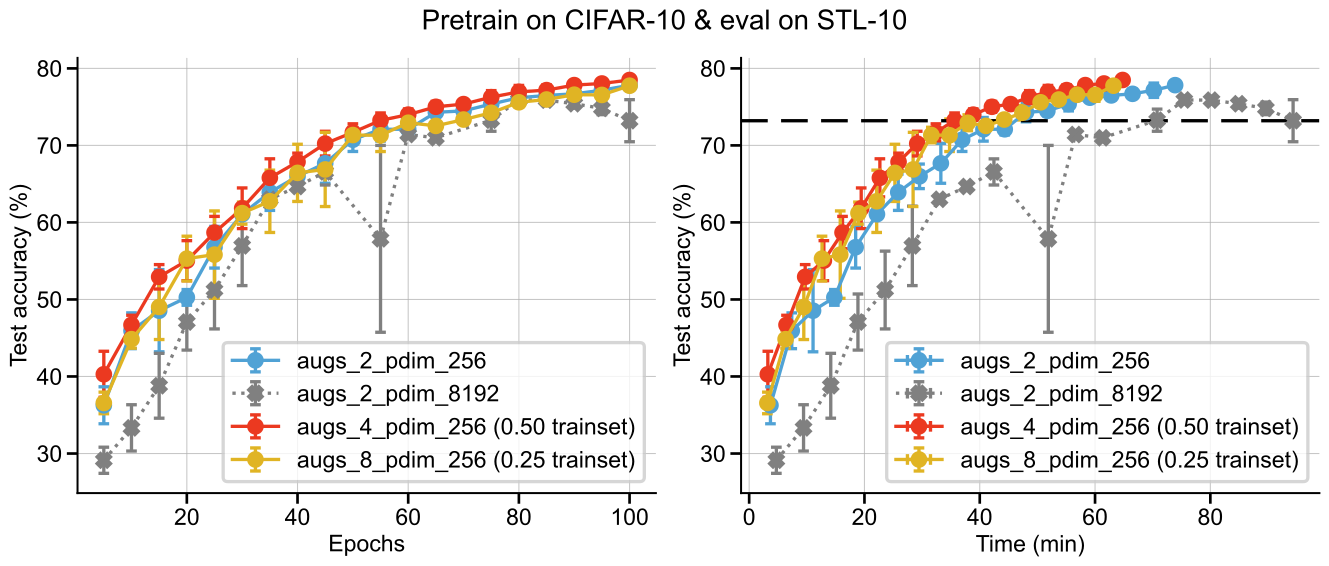}
    \caption{VICReg loss pretraining on fraction of CIFAR-10 trainset, linear evaluation on STL-10 labelled set.}
    \label{fig:fig5A_CIFAR_VICReg}
\end{figure}

\begin{figure}[!htbp]
    \centering
    \includegraphics[width=0.8\linewidth]{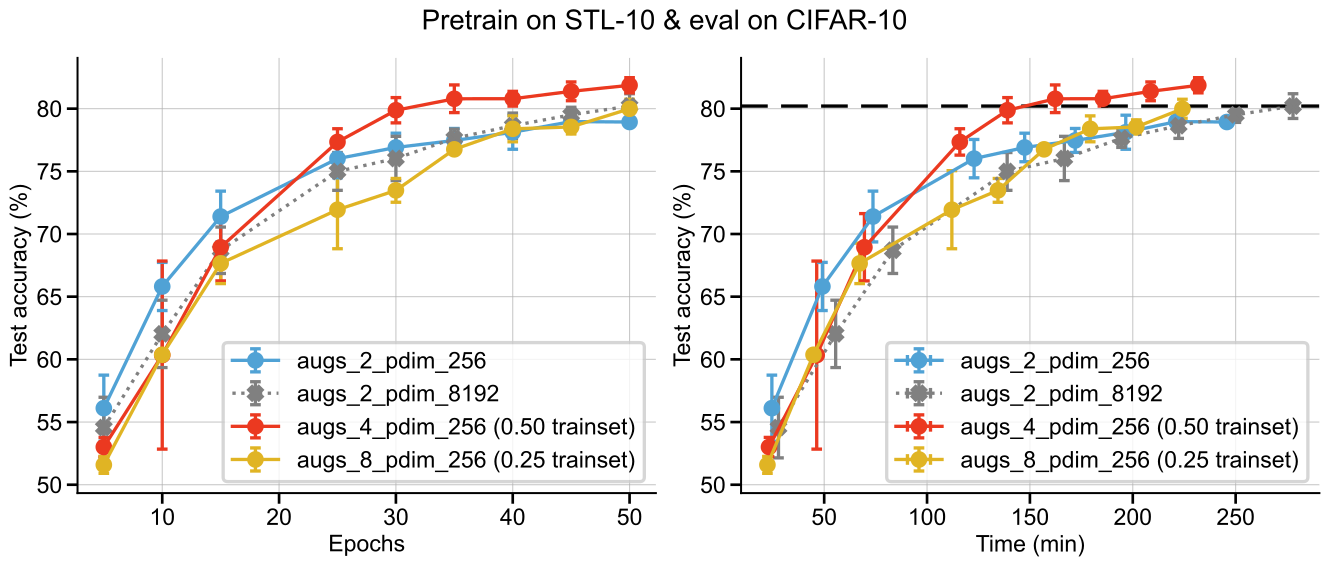}
    \caption{BarlowTwins loss pretraining on fraction of STL-10 unlabelled set, linear evaluation on CIFAR-10 train set.}
    \label{fig:fig5A_STL_Barlow}
\end{figure}
\clearpage

\section{Additional Experiments probing multi-augmentation learning}
% From the above table, it’s clear that increasing the number of augmentations improves downstream performance in line with our claim. However, using more augmentations also increases the per-epoch time, thereby increasing the overall experiment time. In other words, a tradeoff exists between the performance and the compute time. This tradeoff is further demonstrated in Fig. 6 of our manuscript, where we present a Pareto frontier of performance vs compute time by varying the number of augmentations and the fraction of the dataset. Therefore, given our computing budget, we have restricted our experiments to 8 augmentations.

\subsection{Longer Pretraining to determine early stopping}
\label{sec:longer_pretraining}

\begin{figure}[!htbp]
    \centering
    \includegraphics[width=0.8\linewidth]{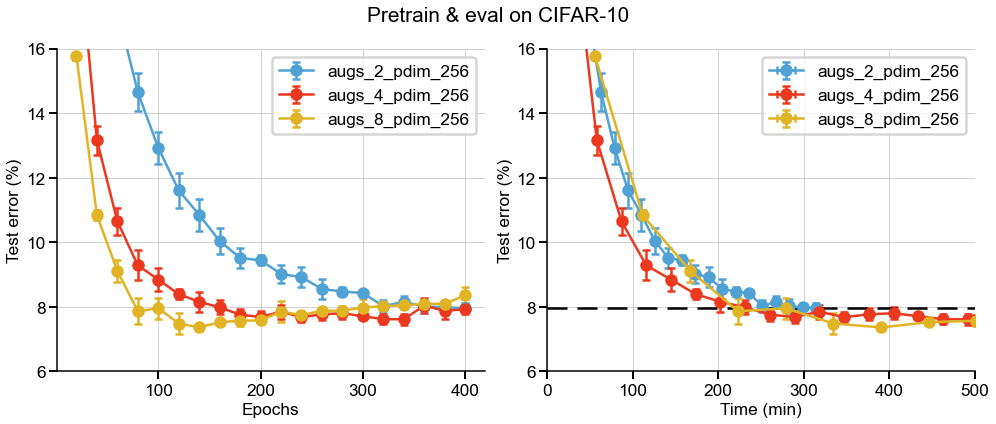}
    \caption{BarlowTwins pretraining on full CIFAR-10 dataset for 400 epochs.}
    \label{fig:rebutt_AjHt_2}
\end{figure}

\begin{table}[ht]
\centering
\begin{tabular}{|l|c|c|}
\hline
{\bf Algorithm }                            & {\bf Best accuracy}        & {\bf Best accuracy @ epoch }       \\ \hline
Barlow-Twins (2-augs) w/ pdim=256   & 92.04 +/- 0.16  & 400 \\ \hline
Barlow-Twins (4-augs) w/ pdim=256   & 92.39 +/- 0.17  & 340 \\ \hline
Barlow-Twins (8-augs) w/ pdim=256   & 92.64 +/- 0.10  & 140 \\ \hline
\end{tabular}
\caption{BarlowTwins pretraining on full CIFAR-10 dataset at 400 epochs (with early stopping)}
\label{fig:table_cifar10_ep400}
\end{table}

\subsection{SwAV-like augmentations for compute efficient multi-augmentation framework}
\label{sec:swav_pretraining}
\begin{figure}[ht]
    \centering
    \includegraphics[width=0.8\linewidth]{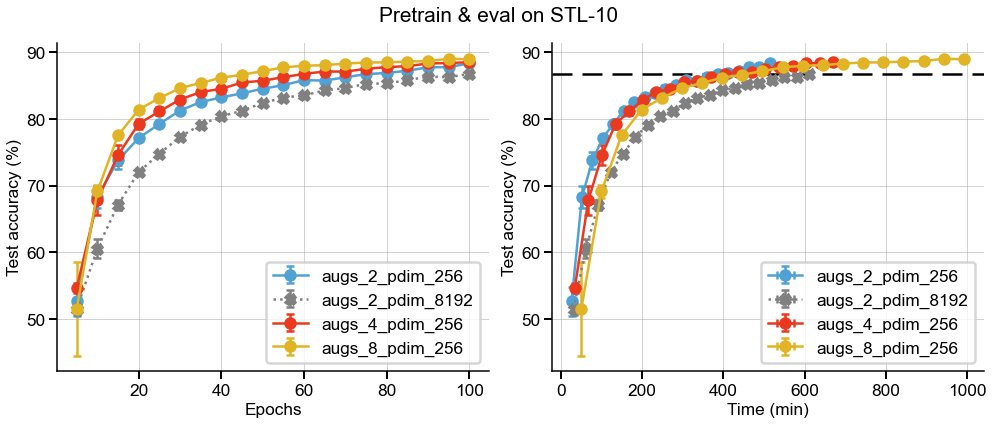}
    \caption{BarlowTwins pretraining on full STL-10 dataset for 100 epochs using SwAV-like augmentations. Specifically, the 2-augmentations setting uses two views that are $64 \times 64$, whereas the 4 (or 8) augmentation setting uses additional two (or six) augmentations that are $32 \times 32$.}
    \label{fig:rebutt_AjHt_1}
\end{figure}

\newpage

\subsection{Training with full dataset with 4/8 augmentations}
\label{sec:full_dset_pretraining}
% {\bf } \\
% We thank the reviewer for this valuable suggestion. Accordingly, we add new experiments with 8 augmentations with Barlow-Twins on the full CIFAR-10 dataset.

\begin{figure}[!htbp]
    \centering
    \includegraphics[width=0.8\linewidth]{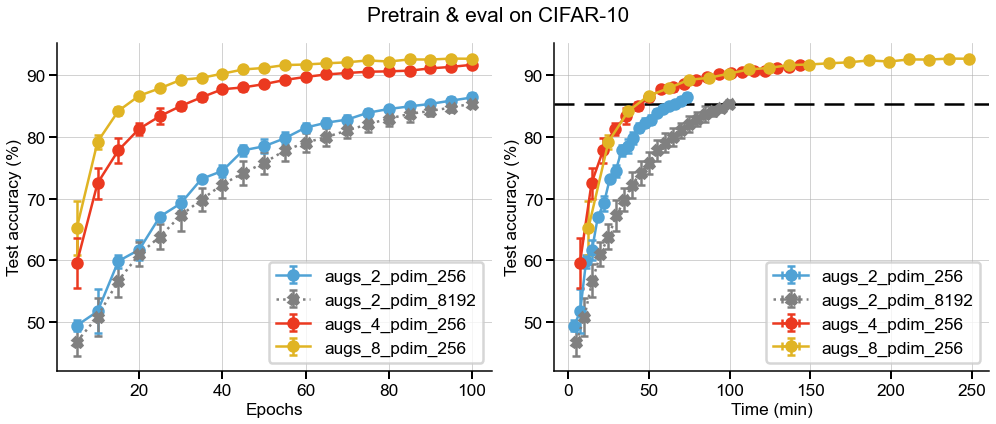}
    \caption{BarlowTwins pretraining on full CIFAR-10 dataset with 2, 4 and 8 augmentations.}
    \label{fig:rebutt_tALx1}
\end{figure}

\begin{table}[ht]
\centering
\begin{tabular}{|l|c|c|c|}
\hline
{\bf Algorithm }                            & {\bf \#augs=2}        & {\bf \#augs=4 }       & {\bf \#augs=8}        \\ \hline
Barlow-Twins w/ pdim=256   & 86.43 +/- 0.72  & 91.73 +/- 0.16  & 92.71 +/- 0.19  \\ \hline
Barlow-Twins w/ pdim=8192  & 85.44 +/- 0.54  & 91.40 +/- 0.32  & 92.40 +/- 0.13  \\ \hline
\end{tabular}
\caption{BarlowTwins pretraining on full CIFAR-10 dataset at 100 epochs}
\label{fig:table_cifar10}
\end{table}

\end{document}